\documentclass[a4paper,11pt]{article}
\usepackage[utf8]{inputenc}
\usepackage[T1]{fontenc}
\usepackage{hyperref}
\usepackage{fullpage}
\usepackage{url}
\usepackage{booktabs}
\usepackage{nicefrac}
\usepackage{microtype}
\usepackage{amssymb}
\usepackage{amsbsy}
\usepackage{amsmath}
\usepackage{amsthm}
\usepackage{amsfonts}
\usepackage{mathtools}
\usepackage{bbm}
\usepackage{latexsym}
\usepackage{graphics}
\usepackage{graphicx}
\usepackage{epstopdf}
\usepackage{caption}
\usepackage{subcaption}
\usepackage{latexsym}
\providecommand{\U}[1]{\protect\rule{.1in}{.1in}}
\newtheorem{thm}{Theorem}
\newtheorem{lemma}{Lemma}

\newtheorem{remark}{Remark}
\newtheorem{prop}{Proposition}
\DeclareMathOperator*{\argmin}{arg\,min}
\DeclareMathOperator*{\argmax}{arg\,max}
\begin{document}

\title{Optimal Transport Relaxations \\with Application to Wasserstein GANs}
\author{Saied Mahdian, Jose Blanchet, Peter Glynn\\Department of Management Science and Engineering, Stanford University}
\date{}
\maketitle

\begin{abstract}
We propose a family of relaxations of the optimal transport problem which
regularize the problem by introducing an additional minimization step over a
small region around one of the underlying transporting measures. The type of
regularization that we obtain is related to smoothing techniques studied in
the optimization literature. When using our approach to estimate optimal
transport costs based on empirical measures, we obtain statistical learning
bounds which are useful to guide the amount of regularization, while
maintaining good generalization properties. To illustrate the computational
advantages of our regularization approach, we apply our method to training
Wasserstein GANs. We obtain running time improvements, relative to current
benchmarks, with no deterioration in testing performance (via FID). The
running time improvement occurs because our new optimality-based threshold
criterion reduces the number of expensive iterates of the generating networks,
while increasing the number of actor-critic iterations.
\end{abstract}

\section{Introduction}

Optimal transport costs, which include the Wasserstein Distance and the
Earth-Mover-Distance as special cases, have become useful tools in machine
learning and statistics
\cite{kolouri2017optimal,arjovsky2017wasserstein,abadeh2015distributionally,kusner2015word,cuturi2013sinkhorn,blanchet2019quantifying}.
The optimal transport cost between two distributions is computed (in its
primal form) as a minimization problem, in which the cost of transporting one
distribution to another is minimized over all possible joint distributions,
leading to linear program (see for example,\cite{villani2003topics}).

Optimal transport provides great flexibility when comparing (probability)
measures and histograms. The transportation cost function (which we refer to
as the cost function) can be used to capture key geometric characteristics
\cite{kolouri2017optimal}. It can be also used to compare discrete vs
continuous distributions directly, without introducing smoothing, in contrast
to alternatives such as the Kullback-Leibler divergence (see
\cite{arjovsky2017wasserstein,genevay2016stochastic} for more details). Also,
by judiciously choosing the cost function, a Wasserstein distance can generate
either the topology corresponding to weak convergence or the total
variation distance.

In data-driven applications, one needs to estimate the optimal transport cost
by means of sampled data. This involves using an empirical measure, say
$\mu_{n}$, as a surrogates for the underlying population probability measure,
say $\mu_{\infty}$. However, this direct approach fails to recognize that
empirical measures are just imperfect descriptions of the underlying
probabilities, and a small amount of perturbation in the empirical measure
also may yield equally valid descriptions of the underlying probabilities.
Adopting this perspective is particularly important in light of the fact
that non-parametric empirical estimators of the Wasserstein distance converge
slowly (at rate $O\left(  n^{-1/d}\right)  )$ where $d$ is the underlying
dimension of the distribution and $n$ the number of samples, see
\cite{dudley1969speed,weed2017sharp}. It is natural to take the view that
plausible variations of the data can be used to facilitate the estimation of
optimal transport costs.

Using this insight, we provide a relaxation which regularizes the optimal
transport cost between, $\mu_{n}$ and $\mu$, say $D_{\widetilde{c}}\left(
\mu_{n},\mu\right)  $ (depending on some cost function $\widetilde{c}\;$).
Our relaxation takes the generic form
\begin{equation}
G_{\delta}\left(  \mu_{n},\mu\right)  =\inf\{D_{\widetilde{c}}\left(
\nu,\mu\right)  :\nu\in\mathcal{D}_{\delta}\left(  \mu_{n}\right)  \},\label{M}
\end{equation}
where $\mathcal{D}_{\delta}\left(  \mu_{n}\right)  $ is a suitable region of
`size' $\delta$ around $\mu_{n}$. The region $\mathcal{D}_{\delta}\left(
\mu_{n}\right)  $ will typically\ be defined in terms of an optimal transport
neighborhood of size $\delta$ around $\mu_{n}$, namely $\mathcal{D}_{\delta
}\left(  \mu_{n}\right)  =\{\nu :D_{c}\left(  \mu_{n},\nu\right)  \leq\delta\}$,
for some optimal transport cost $D_{c}\left(  \mu_{n},\nu\right)  $ depending on
a cost function $c$. So, as $\delta\rightarrow0$, we recover the standard
optimal transport cost. 

We stress that $\mathcal{D}_{\delta}\left(  \mu_{n}\right)  $ could be
defined using criteria other than optimal transport, but given the flexibility
mentioned earlier, we focus on optimal transport neighborhoods as stated
earlier. We can also introduce a neighborhood around both $\mu_{n}$ and
$\mu$ to define the inf. This modification can also be studied with the methods
that we present. 

The map $\mu_{n}\mapsto G_{\delta}\left(  \mu_{n},\mu\right)  $ is
intuitively a more regular object than $G_{0}\left(  \mu_{n},\mu\right)
=D_{\widetilde{c}}\left(  \mu_{n},\mu\right)  $ as it is less sensitive to
small perturbations of $\mu_{n}$. Of course, this type of regularity is also
achieved by maximizing over a neighborhood of $\mu_{n}$ (instead of minimizing),
but this operation leads to computational complications because the optimal
transport cost is a convex functional. 
The dual formulation of the optimal transport costs can be used to connect our
relaxation, at least formally, to smoothing techniques that are often used in
the non-smooth convex optimization literature \cite{nesterov2005smooth}.

As indicated earlier, the regularization approach that we take is particularly
meaningful given the slow rates of convergence in the empirical estimation of
Wasserstein distances. Moreover, since the estimated Wasserstein distance is a positive random variable,
the statistical error is likely to often have a right-tail bias, thus the
minimization operation that we apply in (\ref{M}) to regularize the
Wasserstein distance is also sensible as a means of mitigating this bias. However, we
need to be careful to not overcompensate. So, we also provide statistical
learning bounds which can be used to ensure a choice of $\delta$ which enables
the use of $G_{\delta}\left(  \mu_{n},\mu\right)  $, plus a small correction
term, as an upper bound for $D_{\widetilde{c}}\left(  \mu_{\infty},\mu\right)
$. These statistical learning bounds are presented in Theorem
\ref{thm:delta_sel}. The parameter $\delta>0$ could also be chosen by a
cross-validation procedure.

There are other regularization methods to estimate optimal transport costs.
Some of these techniques require some smoothness or absolute
continuity between the measures involved; this occurs, for example, when using
entropic regularization,
\cite{cuturi2013sinkhorn,sanjabi2018convergence,gulrajani2017improved}. Others
impose low rank constraints, as in \cite{forrow2019statistical}, in the
setting of domain adaptation, and others
(\cite{sanjabi2018convergence,gulrajani2017improved}) focus on specific
applications such as Wasserstein GANs.

Our relaxation technique does not require smoothing or low rank
properties. It acts directly at the same level of generality as the original
optimal transport formulation. However, we are able to show that $G_{\delta
}\left(  \mu_{n},\mu\right)  $ can often be evaluated directly and
conveniently in terms of $D_{\widetilde{c}}\left(  \mu_{n},\mu\right)  $,
leading to a variation of the optimal transport cost formulation which can
then be used in conjunction with any of the regularization methods mentioned
earlier. So, we do not see our work as a competitor to these regularization
methods. Our approach can be reasonably viewed a
pre-conditioning step which can be applied before any regularization tool that
uses additional data structure.

As an application of our framework, we introduce a regularized Wasserstein GAN
formulation which takes the form
\[
\min_{\theta}\min_{\mathcal{W}\left(  \mu_{n},\nu\right)  \leq\delta
}\mathcal{W}\left(  \mu_{\theta},\nu\right)  =\min_{\theta}G_{\delta}\left(
\mu_{n},\mu_{\theta}\right)
\]
where $\mathcal{W}\left(  \mu_{\theta},\nu\right)  $ is a Wasserstein distance
between $\mu_{\theta}$ and $v$, and we are choosing $\widetilde{c}=c$, also
coinciding with the metric used to define $\mathcal{W}$. The parameter
$\theta$ represents the design of the generative network. The standard
Wasserstein GAN formulation, \cite{arjovsky2017wasserstein}, is recovered by
setting $\delta=0$. In Section \ref{sec:dro_wgans} we show that under mild
assumptions,
\begin{equation}
\min_{\theta}\min_{\mathcal{W}\left(  \mu_{n},\nu\right)  \leq\delta
}\mathcal{W}\left(  \mu_{\theta},\nu\right)  =\min_{\theta}\left(
\mathcal{W}\left(  \mu_{\theta},\mu_{n}\right)  -\delta\right)  ^{+}.\label{GAN1}
\end{equation}

Therefore, it is easily seen, after taking the gradients, that the number of
iterates of the generative network, parameterized by $\theta$, is reduced
relative to the actor-critic iterates, represented by $f$. While this implementation
device (i.e. iterating the actor critic more often than the generator) is used
in practice to speed up training times, our approach is theoretically
supported from an optimality perspective. The inner minimization problem we
introduce yields an optimal regularization form, which corresponds to
`flattening' the optimization surface in the parameter space $\theta$. The
amount of flattening is governed by $\delta$, which should
correspond to the degree of ambiguity in the data, measured from a statistical
point of view.
In summary, our Optimal Transport Relaxation (OTR) formulation suggests that training of the
generative network can be reduced without loss of performance. We validate our
findings by experimenting on two datasets: MNIST and CIFAR10.

Finally, we comment, owing to a duality result given in Theorem \ref{Thm0},
that $G_{\delta}\left(  \mu_{n},\mu_{\theta}\right)  $ admits an economic
interpretation as a distributionally robust revenue maximization problem in
which an agent wishes to select a pricing policy which is robust to
perturbations in a customer's demand. While we do not exploit this
interpretation directly in this paper, we believe that this formulation is of
independent interest and thus it is worth exposing it in our Introduction.  

The rest of the paper is organized as follows. In Section
\ref{sec:tractability} we introduce the standard optimal transport problem,
together with our novel OTR formulation. We compare the dual of the standard
optimal transport problem to its robustified analogue. We show a strong
duality result in the sense that the robustified optimal transport dual and
primal achieve the same value. Next, using a general duality result in the
distributionally robust optimization literature \cite{blanchet2019quantifying}, we provide a convenient representation for the primal optimal transport
problem. We use this representation to obtain closed form expressions for the
contribution of the artificial player introduced in our distributionally
robust formulation. These closed form expressions, in particular, include
formulation (\ref{GAN1}). In Section \ref{sec:drot_stats}, we discuss
statistical learning bounds which provide generalization guarantees for our
empirical estimator. We argue that our OTR-based estimator is intuitively more
desirable than the standard empirical estimator for optimal transport costs,
because it is directly seen to be smaller than the standard estimator.  Nevertheless,
Theorem \ref{thm:delta_sel} guarantees that it can be used to build an upper
bound for the underlying optimal transport cost with high probability. We then
provide numerical evidence to demonstrate our intuition. Our
numerical examples suggest that our OTR estimator is
often a better upper bound than the standard Wasserstein estimator.
The proofs of all theorems are provided in the Appendix. 

\section{Problem Formulation, Interpretations and
Tractability\label{sec:tractability}}

We start by formulating the standard optimal transport problem. To do so, we
shall introduce notation which will also be useful when describing our
proposed formulation. Throughout the paper we will consider distributions
supported on metric spaces $\mathcal{S}_{X}$ and $\mathcal{S}_{Y}$ with
metrics $d_{X}$ and $d_{Y}$, respectively. We
assume, for simplicity in the exposition that the spaces are complete,
separable and compact.

We shall use $X$ to denote a generic random variables taking values in
$\mathcal{S}_{X}$. Likewise, a generic random variable $Y$ will take values in
$\mathcal{S}_{Y}$. The space of Borel probability measures defined on
$\mathcal{S}_{X}$ and $\mathcal{S}_{Y}$ are defined as $\Pi_{X}$ and $\Pi_{Y}$, respectively. 
We use $\Pi_{X,Y}$ to denote the set of all couplings between
$X,Y$ (i.e. joint Borel probability measures on $\mathcal{S}_{X}\times\mathcal{S}_{Y}$). 
Further, $\Pi_{X,Y}(\mu_{0},\nu)$ is the subset of
$\Pi_{X,Y}$ such that $X\sim\mu_{0}$ and $Y\sim\nu$ (i.e. $X$ follows
distribution $\mu$ and $Y$ follows distribution $\nu$).

Given a generic element $\pi\in\Pi_{X,Y}$, $\pi_{X}$ is the marginal
distribution of $X$ and $\pi_{Y}$ is the marginal distribution of $Y$. So,
$\pi\in\Pi_{X,Y}\left(  \mu_{0},\nu\right)  $ implies that $\pi_{X}=\mu_{0}$ and
$\pi_{Y}=\nu$.

The standard optimal transport problem, also known as the Monge-Kantorovich
problem, can be written as (see \cite{villani2003topics})
\[
\mathcal{P}_{0}:D_{\widetilde{c}}\left(  \mu_{0},\nu\right)  =\min
\{\mathbb{E}_{\pi}\widetilde{c}(X,Y):\pi\in\Pi_{X,Y}(\mu_{0},\nu)\}
\]
where $\widetilde{c}:\mathcal{S}_{X}\times\mathcal{S}_{Y}\rightarrow
\lbrack0,\infty)$ is a lower semi-continuous function. Clearly,
$D_{\widetilde{c}}\left(  \mu_{0},\nu\right)  $ is the solution of a linear
programming problem (albeit, an infinite dimensional one). 
We now consider the corresponding dual. 
First, let $C\left(  \mathcal{S}_{X}\right)  $ and
$C\left(  \mathcal{S}_{Y}\right)  $ be the space of continuous functions on
$\mathcal{S}_{X}$ and $\mathcal{S}_{Y}$, respectively. Next, define 
\[
\mathcal{A}\left(  \widetilde{c}\right)  =\{\left(  \alpha,\beta\right)
\in\mathcal{S}_{X}\times\mathcal{S}_{Y}:\alpha(x)+\beta(y)\leq\widetilde{c}(x,y)\text{ for all }x\in\mathcal{S}_{X},y\in\mathcal{S}_{Y}\},
\]
then, the dual problem formulation of $\mathcal{P}_{0}$ is
\[
\mathcal{\bar{P}}_{0}:\sup\{\mathbb{E}_{\mu_{0}}\alpha(X)+\mathbb{E}_{\nu
}\beta(Y):\left(  \alpha,\beta\right)  \in\mathcal{A}\left(  \widetilde{c}\right)  \}.
\]
It is known (see \cite{villani2003topics}) that strong duality holds.

To define our relaxed optimal transport
formulation, we introduce the region
\[
\mathcal{D}_{\delta}\left(  \mu_{0}\right)  =\{\nu:D_{c}\left(  \mu
_{0},\nu\right)  \leq\delta\}.
\]
We employ a lower semi-continuous cost function $c:\mathcal{S}_{X}\times\mathcal{S}_{X}\rightarrow\lbrack0,\infty)$ 
satisfying $c\left(x,x\right)  =0$, so that $\mathcal{D}_{0}\left(  \mu_{0}\right)  =\{\mu_{0}\}$. 
As indicated in (\ref{M}), we are interested in
\[
G_{\delta}\left(  \mu_{0},\nu\right)  =\min\{D_{\widetilde{c}}\left(  \mu
,\nu\right)  :\mu\in\mathcal{D}_{\delta}\left(  \mu_{0}\right)  \}.
\]
We have replaced the $\inf$ in (\ref{M}) by $\min$ because
 $\mathcal{D}_{\delta}\left(  \mu_{0}\right)  $ is a compact set in
the weak convergence topology (Prohorov's theorem) and the optimal transport cost, as the supremum of
linear and continuous functionals (by duality), is lower semicontinuous. 

In terms of the dual problem $\mathcal{\bar{P}}_{0}$, our relaxed
formulation then takes the form

\[
G_{\delta}\left(  \mu_{0},\nu\right)  =\min_{\mu\in\mathcal{D}_{\delta}\left(
\mu_{0}\right)  }\sup_{\left(  \alpha,\beta\right)
\in\mathcal{A}\left(  \widetilde{c}\right)  }\mathbb{E}_{\mu}\alpha
(X)+\mathbb{E}_{\nu}\beta(Y).
\]
The next result indicates that duality holds in this representation, meaning,
that $\min$ and $\sup$ can be exchanged, this will serve to provide useful
interpretations for $G_{\delta}\left(  \mu_{0},v\right)  $.

\begin{thm}
\label{Thm0}
\begin{align*}
G_{\delta}\left(  \mu_{0},\nu\right)  =\sup_{\left(  \alpha
,\beta\right)  \in\mathcal{A}\left(  \widetilde{c}\right)  }\min_{\mu
\in\mathcal{D}_{\delta}\left(  \mu_{0}\right)  }\mathbb{E}_{\mu}\alpha(X)+\mathbb{E}_{\nu}\beta(Y).
\end{align*}
\label{thm:duality}
\end{thm}

\bigskip

The above theorem can be used to provide a formal interpretation of our
relaxation as a smoothing technique related to Nesterov's smoothing
\cite{nesterov2005smooth}. 
We have
\begin{align}
G_{\delta}\left(  \mu_{0},v\right)  =\sup_{-\alpha,\beta\in\mathcal{A}\left(
\widetilde{c}\right)  }\inf_{\mu\in\mathcal{D}_{\delta}\left(  \mu_{0}\right)
}E_{v}\beta\left(  Y\right)  -E_{\mu}\alpha\left(  X\right)  =\sup
_{-\alpha,\beta\in\mathcal{A}\left(  \widetilde{c}\right)  }\left(  E_{v}%
\beta\left(  Y\right)  -\phi\left(  \alpha;\mu_{0}\right)  \right)  \label{Map}
\end{align}
where $\phi\left(  \alpha;\mu_{0}\right)  =\sup_{\mu\in\mathcal{D}_{\delta
}\left(  \mu_{0}\right)  }E_{\mu}\alpha\left(  X\right)  $ is a convex
function of $\alpha$. The above representation coincides in form with the smoothing
operator technique introduced by Nesterov, see \cite{nesterov2005smooth},
equation (2.2). The resulting smooth mapping in Nesterov's representation is to be considered as a function of $v$, 
namely $v\mapsto G_{\delta}\left(  \mu_{0},v\right)$.

While we believe that it is interesting to study the transformation
(\ref{Map}) in future research for the purpose of smoothing optimal transport
problems, we shall focus on studying $G_{\delta}\left(  \mu_{0},\nu\right)  $. 
Note that controlling the size of $\delta$ will guarantee the validity of
statistical bounds when estimating optimal transport costs from empirical data.

In addition to the smoothing interpretation given by (\ref{Map}), Theorem
\ref{Thm0} also admits an economic interpretation. Consider an agent who
offers a transportation service to two customers. One of them wishes to
transport a pile of sand out of his/her backyard (this pile of sand is
modeled according to distribution $\mu_{0}$, which represents the demand for
the transportation service), while the other customer wishes to cover a sinkhole in
his/her own backyard (the profile of the sinkhole is modeled by distribution
$v$). It would cost $c\left(  x,y\right)  $ to transport mass from
location $x$ to location $y$ if the customers arrange to solve this
transportation problem among themselves. So, the agent would wish to charge a
price $\alpha\left(  x\right)  $ per unit of mass transported from location
$x$ to the first customer, a price $\beta\left(  y\right)  $ per unit of mass
transported from location $y$ to the second customer, and would do so in such
a way that it is cheaper to pay these prices than to pay the cost of
transporting directly without the intervention of the agent, so $\alpha\left(
x\right)  +\beta\left(  y\right)  \leq c\left(  x,y\right)  $. But, of course,
the agent wishes to maximize the total profit and this yields the dual
interpretation for transporting items, encoded by distributions $\mu_{0},\nu$. Theorem
\ref{Thm0} indicates that $G_{\delta}\left(  \mu_{0},v\right)  $ solves a
distributionally robust revenue maximization problem, in which the agent
selects a policy which is robust to perturbations in the shape of the pile of
sand reported by the first customer.

Next, we provide another representation for $G_{\delta}(\mu_{0},\nu)$, which
forms the basis for the design of gradient and subgradient algorithms and further
simplifications. 

\begin{thm}
\begin{align}
G_{\delta}(\mu_{0},\nu)=(-1)\cdot\min_{\lambda\geq0}\quad\left\{
\lambda\delta+\max_{\pi\in\Pi_{W,Z}(\mu_{0},\nu)}\quad\mathbb{E}_{\pi}\left[
h(W,Y,\lambda)\right]  \right\} \label{pb_dro_dual1}
\end{align}
where $h:\mathcal{S}_{X}\times\mathcal{S}_{Y}\times\mathbb{R}_{+}\rightarrow\mathbb{R}$ and $h(w,y,\lambda)=\sup_{x}\left\{  -\widetilde{c}(x,y)-\lambda c(x,w)\right\}  $.
\label{thm:thm1}
\end{thm}

The above result provides further insight into the smoothness properties
introduced by our relaxation technique. For instance, min-max representation
justifies understanding our relaxation as a regularization technique as in \cite{esfahani2018data,blanchet2016robust}.
Also, consider the case
$\mathcal{S}_{X}=\mathcal{S}_{Y}$ and $c=d_{X}$. Then, the function $h(w,z,\lambda
)$ becomes $\lambda$-Lipschitz in the $w$ argument. In particular, for all
$w_{1},w_{2}\in\mathcal{S}_{X}$,

\[
h(w_{1},z,\lambda)-h(w_{2},z,\lambda)\leq\sup_{x}\left\{  \lambda
c(x,w_{2})-\lambda c(x,w_{1})\right\}  \leq\lambda d_{X}(w_{1},w_{2}).
\]
So, Theorem \ref{thm:thm1} implies that solving for $G_{\delta}(\mu_{0},\nu)$ is
equivalent to solving a standard optimal transport problem with measures
$\mu_{0},\nu$ and a cost function that which replaces $\widetilde{c}(x,y)$ by
a cost function which is $\lambda$-Lipschitz in $w$ and $\lambda$ is regularized.

In view of Theorem \ref{thm:thm1}, we define
\[
g(\lambda,\mu_{0},\nu)=\lambda\delta+\max_{\pi\in\Pi_{W,Y}(\mu_{0},\nu)}\quad\mathbb{E}_{\pi}\left[  h(W,Y,\lambda)\right]  .
\]
Thus, $G_{\delta}(\mu_{0},\nu)=-1\cdot\min_{\lambda\geq0}\,g(\lambda,\mu
_{0},\nu)$. The function $h(\cdot)$ is convex in $\lambda$ and subsequently
$g(\cdot)$ is a convex function of $\lambda$. Hence, (\ref{pb_dro_dual1}) is a
convex optimization problem. Moreover, since $\lim_{\lambda\rightarrow\infty
}g(\lambda,\mu_{0},\nu)=\infty$, the optimal solution set for
(\ref{pb_dro_dual1}) is bounded. Next, we provide a result which can be used
as a basis for a subgradient algorithm to compute $G_{\delta}(\mu_{0},\nu)$. 

\begin{thm}
\label{thm_gd}If $\mathcal{S}_X,\mathcal{S}_Y$ are convex subsets of $\mathbb{R}^d$ (for $d \in \mathbb{N}$), and $\widetilde{c},c$ are continuous, and
$\widetilde{c}(\cdot,y)+\lambda c(\cdot,w)$ is a strictly convex function for $\lambda \geq 0,w$ and $y$, then
$h$ is differentiable in $\lambda$, and the left-hand partial derivative of $g(\lambda,\mu_{0},\nu)$ with respect to $\lambda$ is
\[
\delta+\min_{\pi \in \Pi^{\ast}({\lambda})}
\mathbb{E}_{\pi}\left[  \frac{\partial}{\partial\lambda}h(W,Y,\lambda
)\right],
\]
and the right-hand partial derivative of $g(\lambda,\mu_{0},\nu)$ with respect to $\lambda$ is
\[
\delta+ \max_{\pi \in \Pi^{\ast}({\lambda})}
\mathbb{E}_{\pi}\left[  \frac{\partial}{\partial\lambda}h(W,Y,\lambda
)\right],
\]
where $\Pi^{\ast}({\lambda})$ is set of  optimal solutions to the problem
\begin{align}
\max_{\pi}\quad\mathbb{E}_{\pi\in\Pi_{W,Y}(\mu_{0},\nu)}\left[  h(W,Y,\lambda
)\right]\label{gd_thm_subproblem}
\end{align}

\end{thm}

\begin{remark}
Theorem \ref{thm_gd} still holds if the strict convexity condition for $\widetilde{c}(\cdot,y)+\lambda c(\cdot,w)$
is replaced with the condition that 
$\argmin_{x\in\mathcal{S}_X}\{\widetilde{c}(x,y)+\lambda c(x,w)\}$ 
is a singleton for all $\lambda \geq 0,w$ and $y$.
\end{remark}

\begin{remark}
 Function $g$ is differentiable at any point $\lambda$ if and only if 
 the set $\{ \mathbb{E}_{\pi}\left[  \frac{\partial}{\partial\lambda}h(W,Y,\lambda) \right] |\; \pi \in  \Pi^{\ast}({\lambda}) \}$ 
 is a singleton (for more details, see Corollary 4 of \cite{milgrom2002envelope}).
\end{remark}

Theorem \ref{thm_gd} suggests implementing a subgradient method \cite{bertsekas2015convex} to solve
problem (\ref{pb_dro_dual1}).
In particular, at each iteration $t$, using $\lambda_{t-1}$, we
can find $\pi_{\lambda_{t-1}}$ (a member of $\Pi^{\ast}({\lambda}_{t-1})$) and then $\lambda_{t}$. We assume we have access to
an oracle to solve (\ref{gd_thm_subproblem}). Developing efficient methods to
solve optimal transport problems as in (\ref{gd_thm_subproblem}) is a topic of
separate interest which we will not focus in this paper.
Once we arrive at the optimal solution $(\lambda^{*},\pi_{\lambda^{*}})$ (or a
reasonable approximation of the optimal solution), then an optimal mapping
between $y,x$ solving problem (\ref{pb_dro_dual1}) can be constructed as follows.

\begin{enumerate}
\item For each point $y$, map it to a new point $w$ using $\pi_{\lambda^{*}}$.

\item Find $x$ as the solution to the problem $\sup_{x}\left\{  -\widetilde{c}(x,y)-\lambda^{\ast}c(x,w)\right\}  $.
\end{enumerate}

We conclude this section with an example in which $G_{\delta}(\mu_{0},\nu)$
can be substantially simplified.

\paragraph{Example 1: Wasserstein Distances of Order 2}

Let $\widetilde{c}(x,y)=||x-y||^{2}$ and $c(x,w)=||x-w||^{2}$. Then,
\[
-G_{\delta}(\mu_{0},\nu)=\min_{\lambda\geq0}\left\{  \delta\lambda-\left(
\frac{\lambda}{1+\lambda}\right)  \cdot\left(  \min_{\pi\in\Pi_{W,Y}(\mu
_{0},\nu)}\mathbb{E}_{\pi}||W-Y||^{2}\right)  \right\}
\]

Let $H_{0}=\min_{\pi\in\Pi_{W,Y}(\mu_{0},\nu)}\mathbb{E}_{\pi}||W-Y||^{2}$.
Then the optimal $\lambda$ is $\lambda=\left(  \sqrt{\frac{H_{0}}{\delta}}-1\right)  ^{+}$.

\paragraph{Example 2: Wasserstein distance of order 1 and Wasserstein GANs
\label{sec:dro_wgans}}

Let $\mathcal{S}_{X}=\mathcal{S}_{Y}$ with metric $d_X = d_Y = d$. 
For this subsection, let $\widetilde{c}(x,y)=c(x,y)=d(x,y)$ for all
$x,y \in \mathcal{S}_X$.

First, we claim that 

\[
h(w,y,\lambda)=
\begin{cases}
-d(w,y),\quad\lambda>1\\
-\lambda d(w,y),\quad\lambda<1
\end{cases}
.
\]
This can be seen as follows. Let $x_{(w,y)} =\argmax_{x\in\mathcal{S}_X} \{ \ -d(x,y) - \lambda d(x,w) \}$. Then for $\lambda >1$,

\begin{align*}
&  -d(w,y) = -d(w,y) - \lambda\cdot d(w,w) \leq -d(x_{(w,y)},y) - \lambda\cdot d(x_{(w,y)},w) \\
\Leftrightarrow &  \lambda\cdot d(x_{(w,y)},w) \leq d(w,y) - d(x_{(w,y)},y) \leq d(x_{(w,y)},w)\\
\Leftrightarrow &  (\lambda  - 1) d(x_{(w,y)},w) \leq0 \Leftrightarrow x_{(w,y)} = w.
\end{align*}
A similar argument holds for $\lambda <1$.

In addition, since $h(w,y,\lambda)$ is the supremum of Lipschitz functions, it is continuous in $\lambda$.
So, for $\lambda=1$, $h(w,y,\lambda)=-d(w,y)$. 
For $\lambda\geq 1$, the minimum of $g(\lambda,\mu_0,\nu)$ occurs at $\lambda = 1$.
For $\lambda \in [0,1]$, if $G_{0}(\mu_{0},\nu)-\delta \leq 0$, the minimum of $g(\lambda,\mu_0,\nu)$ occurs at $\lambda = 0$; otherwise,
the minimum occurs at $\lambda = 1$. 
So,
\begin{equation}
G_{\delta}(\mu_{0},\nu)=\max\{G_{0}(\mu_{0},\nu)-\delta,0\}.\label{thresh_eq}
\end{equation}
where $G_{0}(\mu_{0},\nu)$ is the order-1 Wasserstein distance
between $\mu_{0},\nu$.

Expression (\ref{thresh_eq}) can be directly applied to training Wasserstein GANs.
For additional background on these types of generative networks, see
\cite{arjovsky2017wasserstein,gulrajani2017improved}. Wasserstein GANs involve
the optimization problem $\min_{\theta}G_{0}(\mu_{n},\mu_{\theta})$ where
$\mu_{n}$ is the empirical measure of a real dataset and $\mu
_{\theta}$ is a parametric probability measure to be constructed using a generative
network. 

Our \emph{OTR for Wasserstein GANs} takes the form
\begin{align}
\min_{\theta}G_{\delta}(\mu_{n},\mu_{\theta})  & =\min_{\theta}\max\{G_{0}(\mu_{n},\mu_{\theta})-\delta,0\} \nonumber  \\
& =\min_{\theta}\max\{\max_{f\in Lip\left(  1\right)  }E_{\mu_{n}}f\left(
X\right)  -E_{\mu_{\theta}}f\left(  X\right)  -\delta,0\} \nonumber  \\
& =\min_{\theta}\max_{f\in Lip\left(  1\right)  }\left(  E_{\mu_{n}}f\left(
X\right)  -E_{\mu_{\theta}}f\left(  X\right)  -\delta\right)  ^{+}  \label{eq:dr_wgan},
\end{align}
where $Lip\left(  1\right)  $ represent the space of 1-Lipschitz functions
with respect to the metric $d\left(  \cdot\right)  $. Note that $\delta=0$
recovers the problem for Wasserstein GANs. Our implementation involves just a
small modification of standard Wasserstein GAN platforms. However, it is
important to choose the regularization parameter $\delta$ carefully. The
next section provides statistical guidance to this effect. 

Solving (\ref{eq:dr_wgan}) requires only a simple augmentation to any
stochastic gradient descent procedure proposed for Wasserstein GANs. In
particular, (\ref{thresh_eq}) implies
\[
\nabla_{\theta}G_{\delta}(\mu_{r},\mu_{\theta})=\nabla_{\theta}G_{0}(\mu
_{r},\mu_{\theta})\mathbbm{1}\left(  G_{0}(\mu_{r},\mu_{\theta})\geq
\delta)\right)
\]
where $\mathbbm{1}(\cdot)$ is the indicator function. So, in a stochastic
gradient descent implementation, $\theta$ should be updated only when
$\mathbbm{1}\left(  G_{0}(\mu_{r},\mu_{\theta})\geq\delta)\right)  $ and the
procedure will be the same as for Wasserstein GANs. Experiment results are
provided in Section \ref{sec:wgan_experiments}.

\section{The Statistics of the OTR Problem\label{sec:drot_stats}}

In the previous section we studied the optimization problem $\min_{\lambda
\geq0}\,g(\lambda,\mu_{0},\nu)$ with respect to any distribution $\mu_{0}$. In
this section, we study statistical guarantees when $\mu_{0}$ is given by an
empirical measure $\mu_{n}$ of i.i.d. observations, so its canonical
representation takes the form
$\mu_{n}\left(  dx\right)  =n^{-1}\sum_{j=1}^{n}\delta_{X_{i}}\left(  dx\right)  ,$
with the $X_{i}$'s being i.i.d. copies of some distribution $\mu_{\infty}$. 
We derive a confidence interval for $G_{0}(\mu_{0},\nu)$ through the use of concentration inequalities.

In this section, we focus on the case where $\mathcal{S}_X = \mathcal{S}_Y$ and $d_X = d_Y$.
In addition, for all $x,y \in\mathcal{S}_{X}$, we set $c(x,y) = d^{k}(x,y)$ where
$k \geq1$.

Suppose $c,\widetilde{c}$ are Lipschitz functions with Lipschitz constants
$L(c)$ and $L(\widetilde{c})$ respectively. As a result, $h(w,y,\lambda)$ is
Lipschitz in $(w,y)$ with Lipschitz constant $K_{\lambda}=O(L(c)\lambda\vee L(\widetilde{c}))$.

Define
\begin{align*}
\epsilon(n,\rho,\zeta,K_{\lambda}) :=  &  \sqrt{\frac{\log(\frac{1}{\rho})}{2n} } + 4\zeta K_{\lambda} +\\
&  \frac{8\sqrt{2}K_{\lambda}}{\sqrt{n}}\int_{\zeta/4}^{4diam(\mathcal{S}_{X})} \sqrt{ \mathcal{N}(\mathcal{S}_{X},d_{X},\xi/4) \log\left(
2\left\lceil \frac{2diam(\mathcal{S}_{X})}{\xi} \right\rceil +1 \right)
}d{\xi}
\end{align*}
where $\mathcal{N}(\mathcal{S}_{X},d_{X},\xi)$ is the $\xi$-covering number
for ($\mathcal{S}_{X},d_X$).

\begin{thm}
For $k=1,d \geq2, \zeta>0, \delta\geq0, \lambda> L(\widetilde{c})$ with
probability at least $1-\rho$,
\begin{align*}
G_{0}(\mu_{\infty}, \nu) \leq G_{\delta}(\mu_{n}, \nu) + \epsilon
(n,\rho,\zeta,K_{\lambda}) + \lambda\delta
\end{align*}
Also for $k>1, d \geq2, \zeta>0, \delta>0, \lambda= \delta^{-\frac{k-1}{k}}$
with probability at least $1-\rho$,
\begin{align*}
G_{0}(\mu_{\infty}, \nu) \leq G_{\delta}(\mu_{n}, \nu) + \epsilon
(n,\rho,\zeta,K_{\lambda}) + \left(  2\cdot L^{\frac{k}{k-1}}(\widetilde{c}) + 1
\right)  \delta^{\frac{1}{k}}
\end{align*}

\label{thm:delta_sel}
\end{thm}

\begin{remark}
Theorem \ref{thm:delta_sel} also holds when $\mu_n$ and $\mu_{\infty}$ are switched.
As a result, we obtain a confidence interval for $G_{0}(\mu_{\infty}, \nu)$.  
In particular, with probability at least $1 - 2\rho$, $G_{0}(\mu_{\infty}, \nu)$
resides in an interval centered at $G_{\delta}(\mu_{n}, \nu)$ with radius
$\epsilon(n,\rho,\zeta,K_{\lambda}) + \lambda\delta$ for $k=1$ and radius 
$\epsilon(n,\rho,\zeta,K_{\lambda}) + \left(  2\cdot L^{\frac{k}{k-1}}(\widetilde{c}) + 1\right)  \delta^{\frac{1}{k}}$
for $k>1$.

\end{remark}

\begin{remark}
 By optimizing the upper bound in Theorem \ref{thm:delta_sel}, we were able to recover the term $\frac{1}{n^{1/d}}$
 (curse of dimensionality) \cite{dudley1969speed}. For more details, see the appendix for Theorem \ref{thm:delta_sel}.
\end{remark}

\section{Experiments}

\subsection{OTR Wasserstein GAN}
\label{sec:wgan_experiments}

This section provides experiment results evaluating OTR Wasserstein GANs
(WGANs) in Section \ref{sec:tractability}.

We trained on two dataset: MNIST \cite{lecun1998mnist} and CIFAR10
\cite{krizhevsky2009learning}. For the GAN implementation, we used the code
provided by \cite{gulrajani2017improved} and for Frechet Inception Distance (FID) calculation we used the
code provided by \cite{heusel2017gans}. For every fixed initial weights (seed),
we trained our proposed GAN with different values of $\delta$. We performed training 
for 20000 generator iterations.
For CIFAR10,
$\delta \in \{0,1.9,2.0,2.1 \}$. For MNIST, $\delta \in \{0,0.2,0.3,0.4 \}$.

Representative results are provided in Figures \ref{fig::fig1},\ref{fig::fig2} in log-log scale. 
Additional results that use different initializations are provided in the Appendix.
Our experiments indicate that with an `appropriate' choice of $\delta$, OTR WGAN
has a similar test loss performance to WGAN. Also on the CIFAR10 dataset, they
have similar Inception Score \cite{salimans2016improved} performance.
Moreover, OTR WGAN has either the same or faster FID\cite{heusel2017gans}
convergence rate than WGAN. In addition, OTR WGAN trains faster than WGAN
because it skips training the Generator when the threshold criteria is not
met. The `appropriate' values for $\delta$ were found using cross validation.
This `appropriate' value for $\delta$ should be slightly greater than
$\min_{\theta}\mathcal{W}(\mu_{n},\mu_{\theta})$. For values of $\delta$
considerably larger than $\min_{\theta}\mathcal{W}(\mu_{n},\mu_{\theta})$,
training of OTR WGAN is faster; however, the FID performance of OTR WGAN is
worse than WGAN. On the other hand for values of $\delta$ considerably less
than $\min_{\theta}\mathcal{W}(\mu_{n},\mu_{\theta})$, the thresholding
becomes ineffective and OTR WGAN behaves similar to WGAN. In addition, trying
many different initial points (seeds) indicate that OTR WGAN is more stable
and has less volatility compared to WGAN.

\begin{figure}[ptbh]
\centering
\begin{subfigure}[b]{0.45\textwidth}
\includegraphics[width=\textwidth]{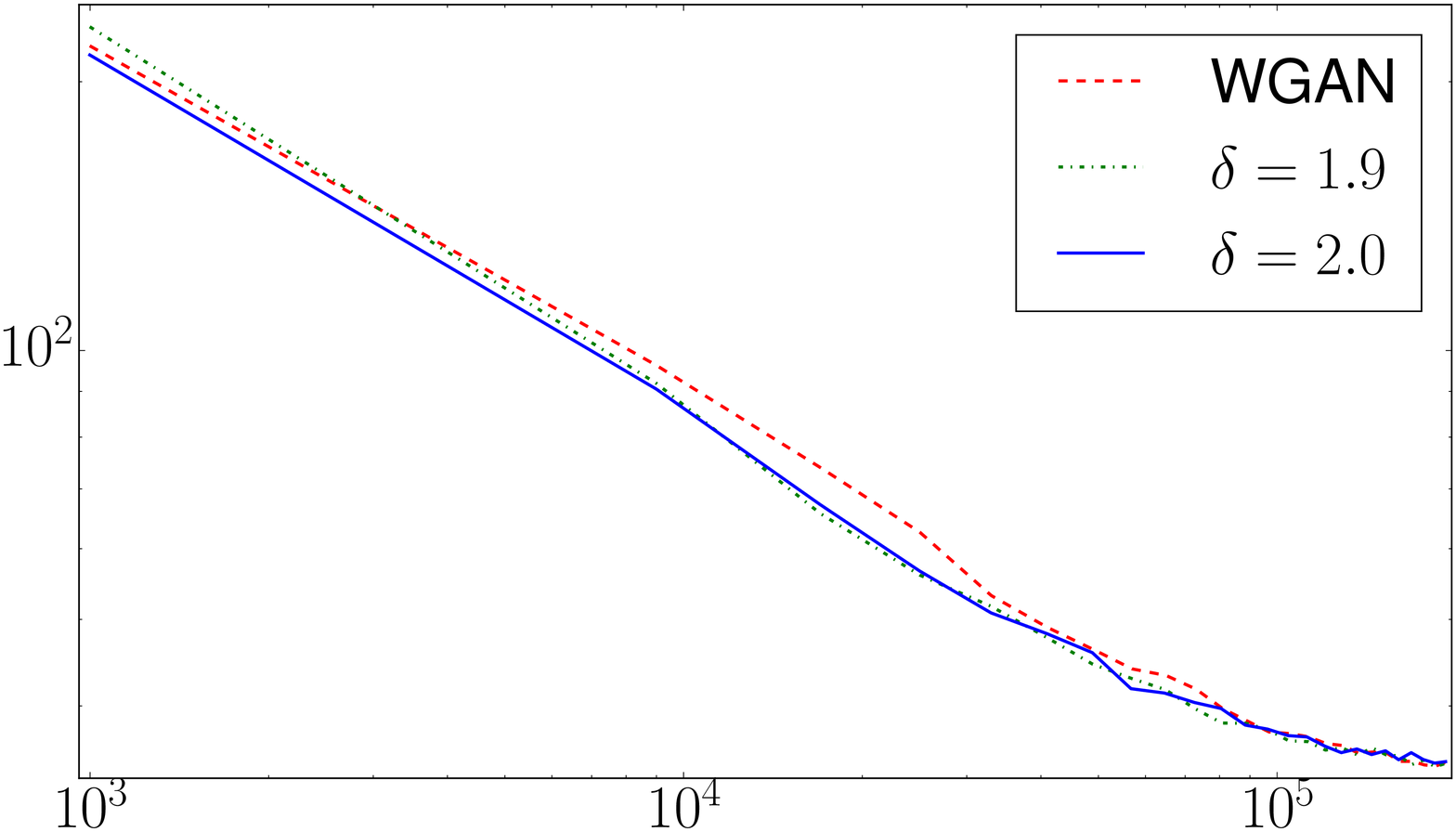}
\caption{}
\end{subfigure}
~ \begin{subfigure}[b]{0.45\textwidth}
\includegraphics[width=\textwidth]{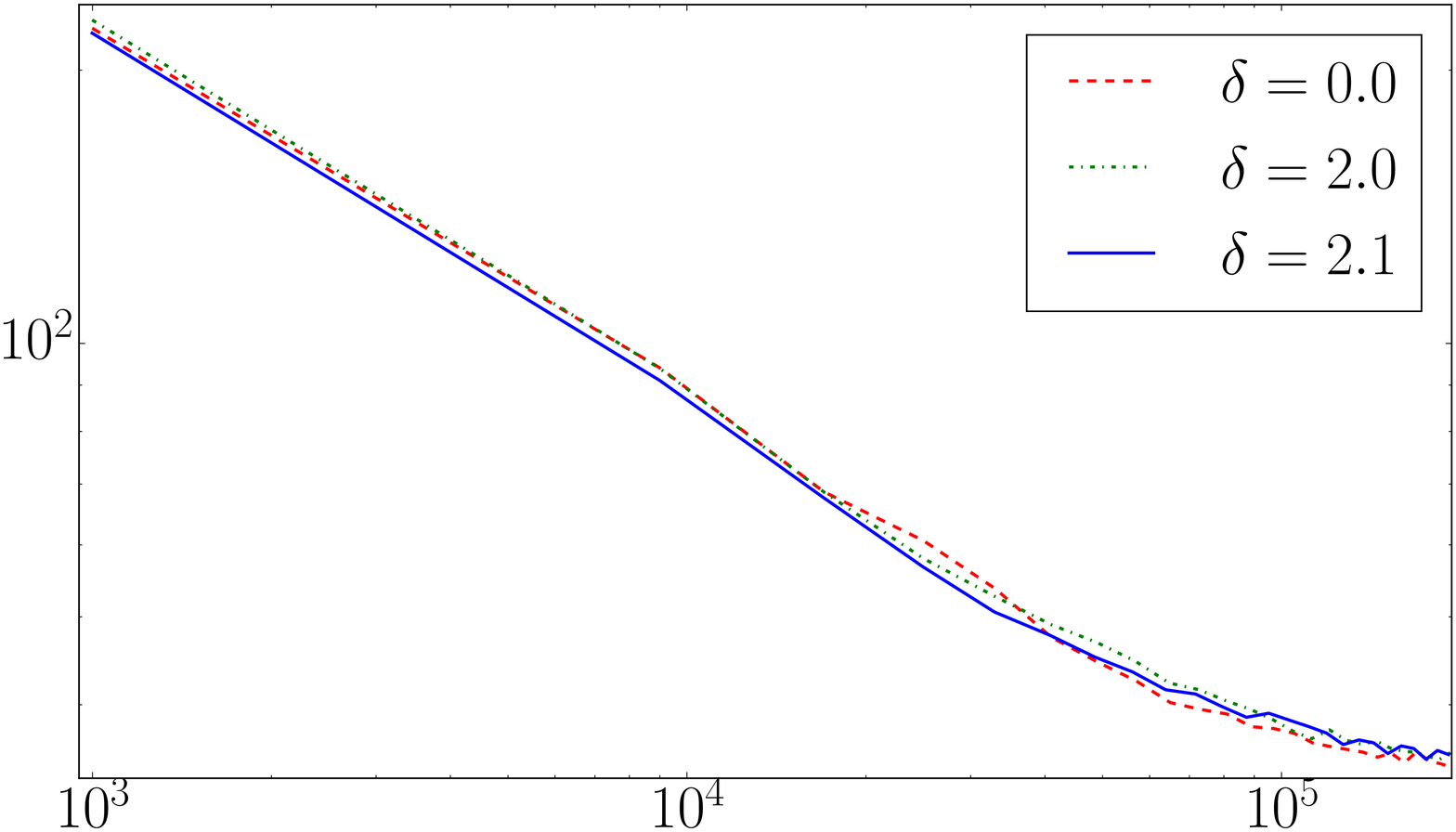}
\caption{}
\end{subfigure}
\caption{FID versus generator iteration on CIFAR10 for comparison of OTR WGAN
and WGAN. Each subfigure presents a different initialization. In the legend,
$\delta$ is the OTR WGAN parameter.}
\label{fig::fig1}
\end{figure}

\begin{figure}[ptbh]
\centering
\begin{subfigure}[b]{0.45\textwidth}
\includegraphics[width=\textwidth]{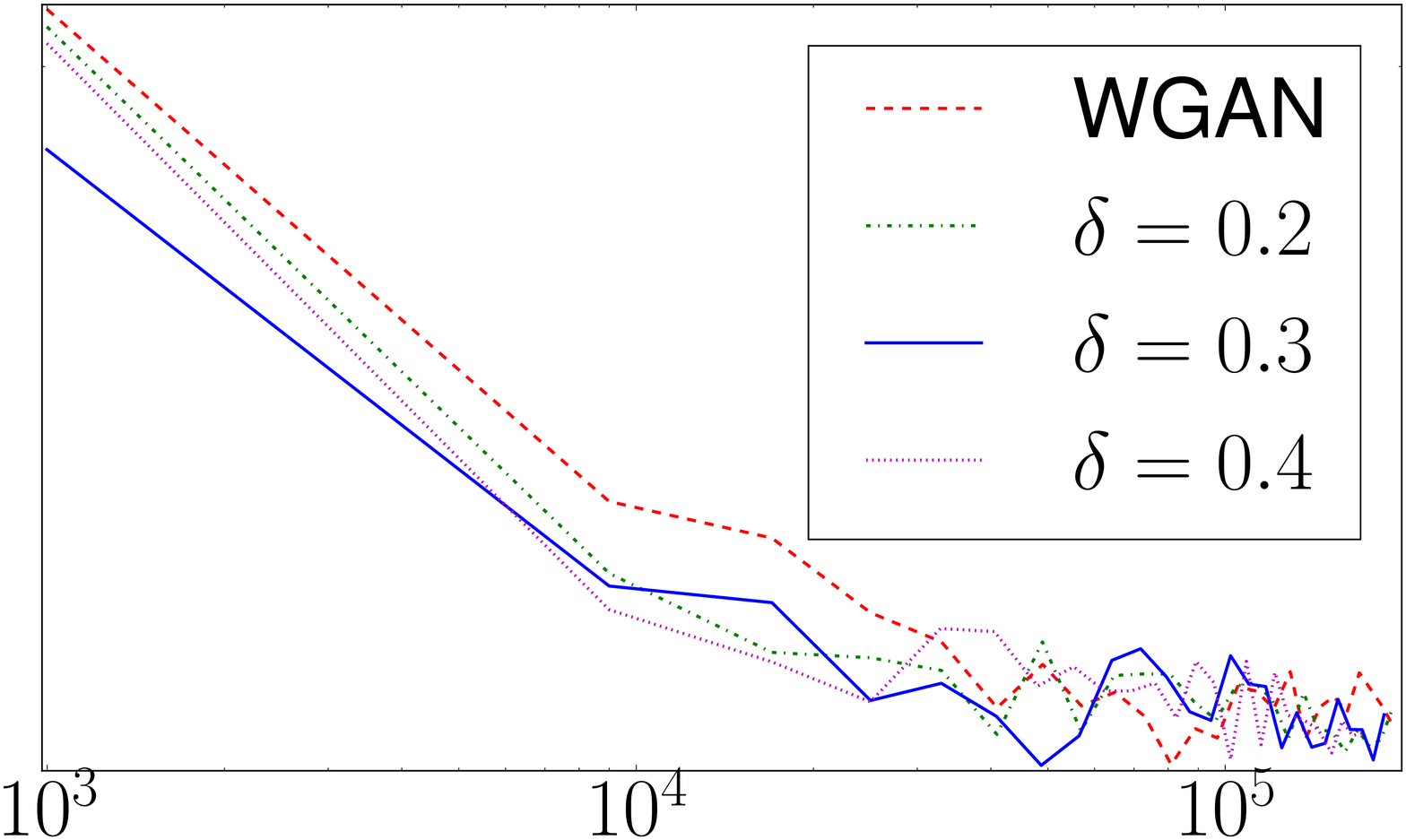}
\caption{}
\end{subfigure}
~ \begin{subfigure}[b]{0.45\textwidth}
\includegraphics[width=\textwidth]{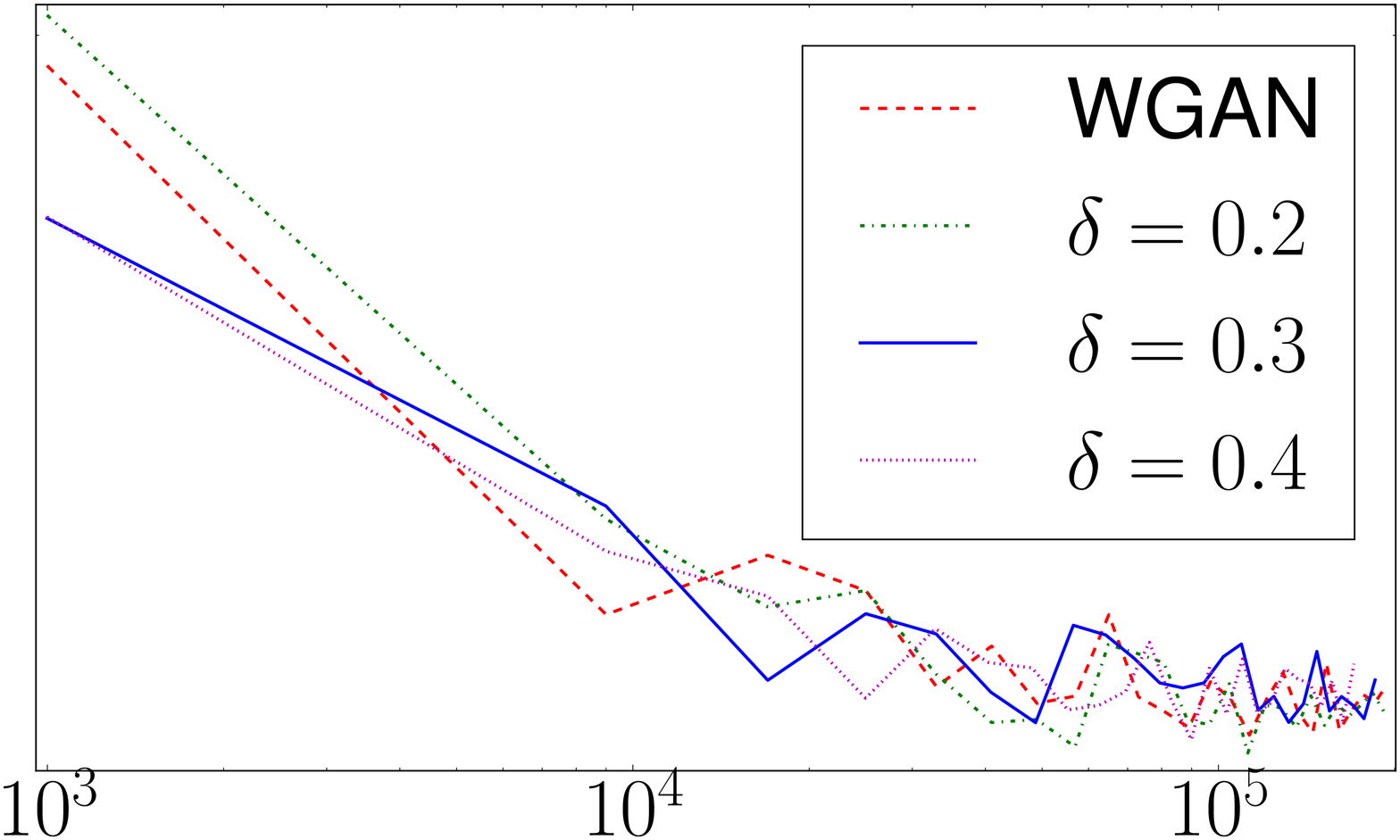}
\caption{}
\end{subfigure}
\caption{FID versus generator iteration on MNIST for comparison of OTR WGAN
and WGAN. Each subfigure presents a different initialization. In the legend,
$\delta$ is the OTR WGAN parameter.}
\label{fig::fig2}
\end{figure}

\subsection{Estimating the Optimal Transport Cost}

In this section, we present simulation results denoting the value of
optimal transport relaxation for estimating the Wasserstein distance
between measures.

Let $\mu_{0}, \nu$ be two probability measures defined on $\mathbb{R}^{20}$.
The measure $\nu$ is constructed from $300$ i.i.d samples of $\mathcal{N}
\left(  0,I_{20 \times20} \right)  $ where $I_{20 \times20}$ is the identity
matrix. The measure $\mu_{0}$ is also constructed from $300$ i.i.d. sampling
of a random vector $X \in\mathbb{R}^{20}$ defined as follows. For each
component $X_{i}$ of $X$ ($1 \leq i \leq20$), $X_{i} := \rho R_{i} + (1 -
\rho^{2})^{\frac{1}{2}} T$ where $\{R_{i}\}_{i=1}^{20},T$ are i.i.d. and
$\mathcal{N} \left(  0,1 \right)  $. In particular, $0 \leq\rho\leq1$
specifies dependence of the $X_{i}$'s.

For $n\in\mathbb{N}$, let $\hat{\mu}_{n}$ be an empirical probability measure
constructed from $n$ i.i.d. samples from $\mu_{0}$. For $\widetilde{c} = c =
\Vert\cdot\Vert^{2}_{2}$, we compute the values of $G_{\delta_{n}}(\hat{\mu
}_{n}, \nu), G_{0}(\hat{\mu}_{n}, \nu)$ where $\delta_{n} = \frac{1}{n^{0.45}}$. 
Then we compare them with the value of $G_{0}({\mu}_{0}, \nu)$.
To solve the optimal transport problems, the implementation from the Python
Optimal Transport Library \cite{flamary2017pot} was used.

Our experiments indicate for large enough values of $n$, the empirical cost
functions $G_{0}(\hat{\mu}_{n}, \nu), G_{\delta_{n}}(\hat{\mu}_{n}, \nu)$
often incur upward shifts relative to $G_{0}({\mu}_{0}, \nu)$. This is
illustrated in \ref{fig::fig3}. Figure \ref{fig:low_rho} and Figure
\ref{fig:high_rho} correspond to high and low dependence of the $X_{i}$'s, respectively.

\begin{figure}[ptbh]
\centering
\begin{subfigure}[b]{0.48\textwidth}
\includegraphics[width=\textwidth]{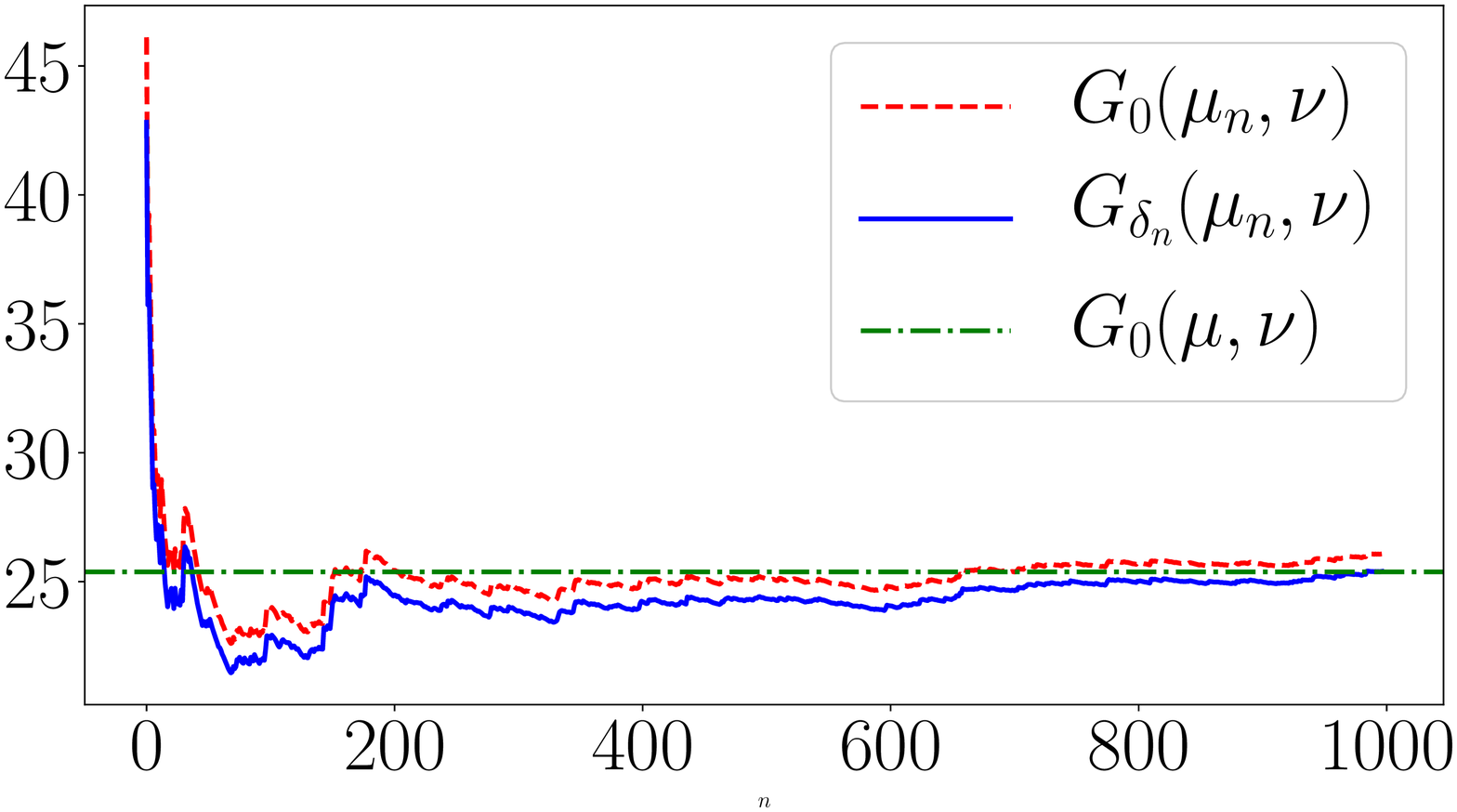}
\caption{$\rho=0.3$}
\label{fig:low_rho}
\end{subfigure}
~ \begin{subfigure}[b]{0.48\textwidth}
\includegraphics[width=\textwidth]{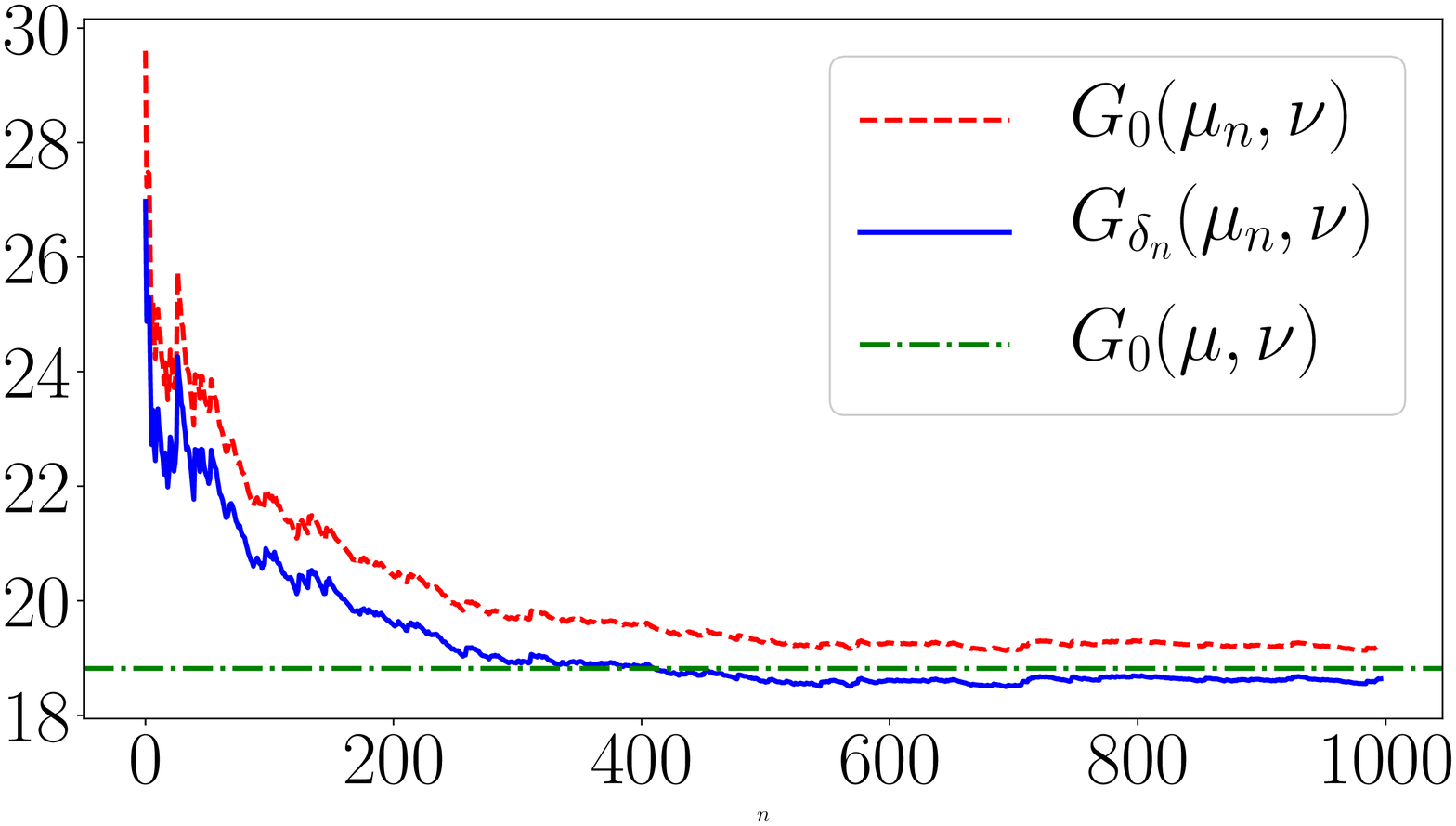}
\caption{$\rho=0.8$}
\label{fig:high_rho}
\end{subfigure}
\caption{Estimating the Optimal Transport Cost}
\label{fig::fig3}
\end{figure}

\bibliographystyle{plain}
\bibliography{Papers}

\clearpage\onecolumn\appendix

\section{Proof of Theorem \ref{thm:duality}}

\label{apx_prf_thm_duality}

The result follows directly from Sion's min-max theorem (see
\cite{sion1958general}). First, the set $\mathcal{D}_{\delta}\left(  \mu
_{0}\right)  $ is convex because, by duality, $D_{c}\left(
\mu_{0},\cdot\right)  $ is convex (because since it is the supremum of linear
functionals). Next, since the spaces involved are compact, the set
$\mathcal{D}_{\delta}\left(  \mu_{0}\right)  $ is compact in the weak
convergence topology, by Prohorov's theorem. Furthermore, it is immediate that
the set $\mathcal{A}\left(  \widetilde{c}\right)  $ is convex. Finally, the
objective function is bilinear both in $\left(  \mu,v\right)  $, on one hand,
and $\left(  \alpha,\beta\right)  $ on the other. By definition of weak
convergence, the functional is continuous in the weak convergence topology
since the elements in $\mathcal{A}\left(  \widetilde{c}\right)  $ are both
continuous, and bounded  and the spaces are compact.

\section{Proof of Theorem \ref{thm:thm1}}

\label{apx_prf_thm1}

We have
\[
G_{\delta}(\mu_{0},\nu)=\min_{D_{c}(\mu,\mu_{0})\leq\delta}\quad\min_{\pi
\in\Pi_{X,Y}(\mu,\nu)}\quad\mathbb{E}_{\pi}\widetilde{c}(X,Y).
\]
Given a coupling $\pi\in\Pi_{X,Y}(\mu,\nu)$ we can always have a
coupling between $X$ and $W\sim\mu_{0}$ (by the gluing lemma, see \cite{villani2003topics}).
Therefore, we have that
\begin{align*}
G_{\delta}(\mu_{0},\nu)=\min_{\pi}\quad &  \int\widetilde{c}(x,y)\pi
(dx,dy,dw)\\
\text{s.t.}\quad &  \int c(x,w)\pi(dx,dw)\leq\delta\\
&  \pi_{W}=\mu_{0},\pi_{Y}=\nu
\end{align*}
which is the same as
\begin{align*}
-G_{\delta}(\mu_{0},\nu)=\max_{\pi}\quad &  \int-\widetilde{c}(x,y)\pi
(dx,dy,dw)\\
\text{s.t.}\quad &  \int c(x,w)\pi(dx,dw)\leq\delta\\
&  \pi_{W}=\mu_{0},\pi_{Y}=\nu
\end{align*}
As a result,
\begin{align*}
& -G_{\delta}(\mu_{0},\nu) =\\   &\max_{\pi}\min_{\lambda\geq0,h_{1}\in C\left(
\mathcal{S}_{X}\right)  ,h_{2}\in C\left(  \mathcal{S}_{Y}\right)  }\left[
\int-\widetilde{c}(x,y)\pi(dx,dy,dw)+\int h_{1}(w)\mu_{0}(dw) - \right.   \\ 
&  \left. \int h_{1}(w)\pi(dx,dy,dw)  +\int h_{2}(y)\nu(dy)-\int h_{2}(y)\pi(dx,dy,dw) + \right. \\
& \left. \lambda\left(\delta-\int c(x,w)\pi(dx,dw)\right)  \right]  ,
\end{align*}
Further, Sion's min-max Theorem \cite{sion1958general} is applicable because the value function is both linear in $\pi$ and $(h_1, h_2)$. 
In particular, it is concave in $\pi$ and convex in $(h_1, h_2)$. We then need to argue upper semi-continuity as function of $\pi$ 
and lower semicontinuity as a function of $(h_1,h_2)$. We choose the topology of uniform convergence over the compact sets 
$\mathcal{S}_X$ and $\mathcal{S}_Y$. Continuity then follows easily by Dominated Convergence. 
Now, to show upper semicontinuity as a function of $\pi$, 
we consider the space of probabilities under the weak convergence topology. It suffices to show that $\int\widetilde{c}(x,y)\pi(dx,dy,dw)$ and 
$\int c(x,w)\pi(dx,dw)$ are lower semicontinuous as a function of $\pi$, since the remaining terms involving $\pi$ involve integrals 
of continuous functions over compact sets (hence continuous and bounded functions) and therefore those remaining terms are directly 
seen to be continuous by the definition of weak convergence (denoted by $\Rightarrow$). 
We need to show that if $\pi_n \Rightarrow \pi$ as $n \to \infty$, then $\liminf \int c d\pi_n  \geq \int cd\pi$. 
By the Skorokhod representation, we may assume that there exists $Z_n = (X_n, Y_n, W_n)$ such that $Z_n$ has distribution $\pi_n$ 
and $Z$ having distribution $\pi$, such that $Z_n \rightarrow Z$ almost surely as $n \to \infty$. 
Then, we have that

\begin{align*}
\liminf \int cd\pi_n = \liminf E(c(Z_n)) \geq \int E(\liminf c(Z_n)) \geq E( c(Z)) = \int c\;d\pi,
\end{align*}

where the first inequality follows by Fatou's lemma and the second inequality follows because $c$ is lower semicontinuous. 
A similar argument holds for $\widetilde{c}$.
As a result,
\begin{align*}
& -G_{\delta}(\mu_{0},\nu) = \\ &  \min_{\lambda\geq0,h_{1}\in C\left(
\mathcal{S}_{X}\right)  ,h_{2}\in C\left(  \mathcal{S}_{Y}\right)  }\max_{\pi
}\left[  \int\left(  -\widetilde{c}(x,y)-h_{1}(w)-h_{2}(y)-\lambda
c(x,w)\right)  \pi(dx,dy,dw)\right.  \\
&  \left.  +\lambda\delta+\int h_{1}(w)\mu_{0}(dw)+\int h_{2}(y)\nu(dy)\right]  .
\end{align*}
The above expression implies that for all $x,y,w$ we must have:
\[
-\widetilde{c}(x,y)-h_{1}(w)-h_{2}(y)-\lambda c(x,w)\leq0\Rightarrow\sup
_{x}[-\widetilde{c}(x,y)-\lambda c(x,w)]\leq h_{1}(w)+h_{2}(y)
\]
Therefore,
\begin{align*}
-G_{\delta}(\mu_{0},\nu) &  =\min_{\lambda\geq0}\max_{\pi\in\Pi_{W,Y}(\mu
_{0},\nu)}\left\{  \lambda\delta+\mathbb{E}_{\mu_{0}}h_{1}(W)+\mathbb{E}_{\nu
}h_{2}(Y)\right\}  \\
&  =\min_{\lambda\geq0}\max_{\pi\in\Pi_{W,Y}(\mu_{0},\nu)}\left\{
\lambda\delta+\mathbb{E}_{\pi}\left[  \sup_{x}\{-\widetilde{c}(x,Y)-\lambda
c(x,W)\}\right]  \right\}  \\
&  =\min_{\lambda\geq0}\left\{  \lambda\delta+\max_{\pi\in\Pi_{W,Y}(\mu
_{0},\nu)}\quad\mathbb{E}_{\pi}\left[  \sup_{x}\left\{  -\widetilde{c}(x,Y)-\lambda c(x,W)\right\}  \right]  \right\}
\end{align*}

\section{Proof of Theorem \ref{thm_gd}}

\label{apx_prf_gd_thm}

\noindent
The differentiability of $h(\cdot)$ in $\lambda$ is an immediate result of Corollary 4 of \cite{milgrom2002envelope}.

Define
\begin{align*}
u(\lambda, \mu_{0}, \nu) = \max_{\pi\in\Pi_{W,Y}(\mu_{0},\nu)} \quad &
\mathbb{E}_{\pi} \left[  h(W,Y,\lambda) \right]
\end{align*}
Therefore, $g(\lambda, \mu_{0}, \nu) = \delta\lambda+ u(\lambda, \mu_{0},
\nu)$. In addition, define $f(\pi,\lambda) = \mathbb{E}_{\pi} \left[
h(W,Y,\lambda) \right]  $.\newline

\noindent We need to show $\frac{\partial}{\partial\lambda} u(\lambda, \mu
_{0}, \nu) = \mathbb{E}_{\pi_{\lambda}^{*}} \left[  \frac{\partial}{\partial\lambda} h(W,Y,\lambda) \right]  $. 
For this statement to hold,
according to Corollary 4 of \cite{milgrom2002envelope}, a sufficient condition is as follows.
The set $\Pi(\mu_{0},\nu)$ needs to be compact, $f(\pi,\lambda)$ needs to be
continuous in $\pi$, and $\frac{\partial}{\partial\lambda} f(\pi,\lambda)$
needs to be continuous in $(\pi,\lambda)$. In the remainder of this proof, we
will show that this sufficient condition holds.

The set $\mathcal{S}_{X} \times\mathcal{S}_{Y}$ is compact. Therefore,
Prohorov theorem implies under the weak convergence topology, $\Pi(\mu_{0},\nu)$
is a compact set.

On the other hand, the function $-\widetilde{c}(x,y) - \lambda c(x,w)$ is
continuous in $(x,w,y,\lambda)$ and the supremum 
$\sup_{x} \left\{  -\widetilde{c}(x,y) - \lambda c(x,w) \right\}  $ 
is attained due to the compactness of
$\mathcal{S}_{X},\mathcal{S}_{Y}$. Define $x^{*}(w,y,\lambda)$ to be the
maximizing $x$, which will be unique (because  $\widetilde{c}(\cdot,y)+ \lambda c(\cdot,w)$ is strictly convex).
Hence,
\begin{align*}
\sup_{x} \left\{  -\widetilde{c}(x,y) - \lambda c(x,w) \right\}  = -\widetilde
{c}(x^{*}(w,y,\lambda),y) - \lambda c(x^{*}(w,y,\lambda),w)
\end{align*}
Then, Berge's maximum theorem \cite{aliprantis2006infinite} implies
$h(w,y,\lambda)$ is continuous in $(w,y,\lambda)$ and $x^{*}(w,y,\lambda)$ is
upper hemicontinuous in $(w,y,\lambda)$. Moreover, since $x^{*}$ is a single
valued correspondence, it is continuous in $(w,y,\lambda)$.

Since it is defined on a compact set and continuous, $h(\cdot,\cdot,\lambda)$ is bounded. 
Also, $f(\cdot,\lambda)$ is linear in $\pi$. Therefore, under the weak
convergence topology, $f(\pi,\lambda)$ is continuous in $\pi$.

Moreover,
\begin{align}
&  \frac{\partial}{\partial\lambda} f(\pi,\lambda)\nonumber\\
\overset{(a)}{=} \,  &  \mathbb{E}_{\pi} \left[  \frac{\partial}{\partial\lambda} h(W,Y,\lambda) \right] \nonumber\\
= \,  &  \mathbb{E}_{\pi} \left[  -c(x^{*}(W,Y,\lambda),W)
\right]  \label{statement_dev_cont}
\end{align}

In the above statement, $(a)$ is an immediate result of the fact that
$h(\cdot)$ is convex in $\lambda$ and the monotone convergence theorem
together with the fact that $h(\cdot)$ is differentiable in $\lambda$. 
Since $c,x^{*}$ are continuous, the function $-c(x^{*}(W,Y,\lambda),W)$ 
is continuous in $(W,Y,\lambda)$. For fixed $\lambda$, this
function is bounded since it is defined on a compact set. Thus the bounded
convergence theorem implies (\ref{statement_dev_cont}) is continuous in
$\lambda$. In addition, $\frac{\partial}{\partial\lambda} f(\pi,\lambda)$ is
continuous in $\pi$ under the weak convergence topology. So, $\frac{\partial
}{\partial\lambda} f(\pi,\lambda)$ is continuous in $(\pi,\lambda)$.

\section{Proof of Theorem \ref{thm:delta_sel} and Additional Comments}
\label{apx_prf_thm_delta_sel}

\subsection{Proof of Theorem \ref{thm:delta_sel}}

It can be shown \cite{villani2003topics} that
\[
f(X_{1},\cdots,X_{n}):=\min_{\pi\in\Pi_{W,Y}(\mu_{n},\nu)}\mathbb{E}_{\pi
}\left\{  -h(W,Y,\lambda)\right\}  =\sup_{\alpha(\cdot)\in Lip(K_{\lambda})}\{\mathbb{E}_{\mu_{n}}\alpha(W)+\mathbb{E}_{\nu}\alpha_{\lambda}^{h}(Y)\}
\]
where $X_{1},\cdots,X_{n}$ are the i.i.d samples associated with the empirical
measure $\mu_{n}$. $Lip(K_{\lambda})$ denotes the set of all $K_{\lambda}$-Lipschitz functions $f(\cdot)$ defined on $\mathcal{S}_{X}$ such that
$\min_{x\in\mathcal{S}_{X}}|f(x)|=0$. In addition, $\alpha_{\lambda}^{h}(y):=\sup_{w}\{-h(w,y,\lambda)-\alpha(w)\}$.

\begin{prop}
For all $t>0$,
\begin{align*}
\mathbb{P}\left(  f(X_{1},\cdots,X_{n}) - \mathbb{E}f(X_{1},\cdots,X_{n}) \geq
t\right)  \leq\exp\left(  \frac{-2nt^{2}}{K^{2}_{\lambda}.{diam}^{2}(\mathcal{S}_{X})} \right)
\end{align*}
\label{prop:prop1}
\end{prop}

\begin{proof}
See Appendix \ref{apx_prf_prop1}.
\end{proof}

\begin{prop}
With probability at least $1-\rho$,
\[
\min_{\pi\in\Pi_{W,Y}(\mu_{n},\nu)}\mathbb{E}_{\pi}\left\{  -h(W,Y,\lambda
)\right\}  -\min_{\pi\in\Pi_{W,Y}(\mu_{0},\nu)}\mathbb{E}_{\pi}\left\{
-h(W,Y,\lambda)\right\}  \leq\sqrt{\frac{\log(\frac{1}{\rho})}{2n}}+2R_{n}(Lip(K_{\lambda}))
\]
and
\[
\min_{\pi\in\Pi_{W,Y}(\mu_{0},\nu)}\mathbb{E}_{\pi}\left\{  -h(W,Y,\lambda
)\right\}  -\min_{\pi\in\Pi_{W,Y}(\hat{\mu}_{n},\nu)}\mathbb{E}_{\pi}\left\{
-h(W,Y,\lambda)\right\}  \leq\sqrt{\frac{\log(\frac{1}{\rho})}{2n}}+2R_{n}(Lip(K_{\lambda}))
\]
\label{prop:concentration_2}
\end{prop}

\begin{proof}
 See Appendix \ref{apx_prf_prop_conc_2}.
\end{proof}

Define $q_{k} =
\begin{cases}
\lambda\delta, \quad & k=1\\
\left(  2\cdot L^{\frac{k}{k-1}}(\widetilde{c}) + 1 \right)  \delta^{\frac{1}{k}},
\quad & k>1
\end{cases}
$. Now for the event of interest we have

\begin{align*}
&  \left\{  G_{0}(\mu_{0}, \nu) \leq G_{\delta}(\mu_{n}, \nu) +
\epsilon(n,\delta,\zeta,K_{\lambda}) + q_{k} \right\} \\
=  &  \left\{  q_{k} \geq G_{0}(\mu_{0}, \nu) + \min_{\lambda\geq0}
\left\{  \delta\lambda+ \max_{\pi\in\Pi_{W,Y}(\mu_{n},\nu) } \mathbb{E}_{\pi}
h(W,Y,\lambda) \right\}  - \epsilon(n,\delta,\zeta,K_{\lambda}) \right\}
\\
=  &  \left\{  \exists {\lambda\geq0}: q_{k} \geq\delta\lambda+
\underbrace{\max_{\pi\in\Pi_{W,Y}(\mu_{n}, \nu)} \mathbb{E}_{\pi}
h(W,Y,\lambda) - \max_{\pi\in\Pi_{X,Y}(\mu_{0}, \nu)} \mathbb{E}_{\pi}
h(X,Y,\lambda)- \epsilon(n,\delta,\zeta,K_{\lambda})}_{(I)} \right. \\
&  + \left.  \underbrace{\min_{\pi\in\Pi_{X,Y}(\mu_{0}, \nu)} \mathbb{E}_{\pi}
\widetilde{c}(X,Y) + \max_{\pi\in\Pi_{X,Y}(\mu_{0}, \nu)} \mathbb{E}_{\pi}
h(X,Y,\lambda)}_{(II)} \right\}  .
\end{align*}

In the remainder, we show that the above event occurs with probability at
least $1-\rho$.

\begin{lemma}
Let $(\mathcal{S},d)$ present a compact metric space. Also let $f:\mathcal{S}\rightarrow\mathbb{R}$ be 
an L-Lipschitz function (for $L>0$) (i.e. for $x,y\in\mathcal{S}$, $|f(x)-f(y)|\leq L\cdot d(x,y)$.
For $k\geq1,\lambda>0$, define $y_{x}:=\argmax_{y\in\mathcal{S}}\left\{
f(y)-\lambda\cdot d^{k}(x,y)\right\}  $) . Then for $k=1$ and $\lambda>L$,
\[
y_{x}=x
\]
Also for $k>1$,
\[
d(y_{x},x)\leq\left(  L/\lambda\right)  ^{1/\left(  k-1\right)  }.
\]
\label{approxm_lip}
\end{lemma}

\begin{proof}
See Appendix \ref{apx_prf_lip_lemma}.
\end{proof}

For $(I)$, using Proposition \ref{prop:concentration_2} with probability at least
$1-\rho$:
\begin{align*}
&  \max_{\pi\in\Pi_{X,Y}(\mu_{n}, \nu)} \mathbb{E}_{\pi} h(X,Y,\lambda) -
\max_{\pi\in\Pi_{X,Y}(\mu_{0}, \nu)} \mathbb{E}_{\pi} h(X,Y,\lambda) -
\epsilon(n,\delta,\zeta,K_{\lambda})\\
\leq\;  &  \sqrt{\frac{\log(\frac{1}{\rho})}{2n} } + R_{n}(Lip(K_{\lambda})) -
\epsilon(n,\delta,\zeta,K_{\lambda})\\
\overset{(a)}{\leq} \;  &  0
\end{align*}

where (a) is due to Theorem 18 of \cite{luxburg2004distance}.

For $(II)$:
\begin{align*}
&  \min_{\pi\in\Pi_{X,Y}(\mu_{0}, \nu)} \mathbb{E}_{\pi} \widetilde{c}(X,Y) +
\max_{\pi\in\Pi_{X,Y}(\mu_{0}, \nu)} \mathbb{E}_{\pi} h(X,Y,\lambda) \\ =&
\max_{\pi\in\Pi_{X,Y}(\mu_{0}, \nu)} \mathbb{E}_{\pi} h(X,Y,\lambda) -
\max_{\pi\in\Pi_{X,Y}(\mu_{0}, \nu)} \mathbb{E}_{\pi} \{-\widetilde{c}(X,Y)\}\\
\leq &  \max_{\pi\in\Pi_{X,Y}(\mu_{0}, \nu)} \mathbb{E}_{\pi} \{
h(X,Y,\lambda) + \widetilde{c}(X,Y) \}
\end{align*}

For $k=1$, Lemma \ref{approxm_lip} indicates $h(x,y,\lambda) + \widetilde{c}(x,y) =
0$ for $\lambda> L(\widetilde{c})$ and all $(x,y)$. This concludes the proof
for $k=1$.\newline

For $k>1$, Lemma \ref{approxm_lip} shows that for all $(x,y)$,
\begin{align*}
h(x,y,\lambda) + \widetilde{c}(x,y) \leq L(\widetilde{c}) \left(  \frac{L(\widetilde{c})}{\lambda} \right)  ^{\frac{1}{k-1}} + 
\lambda\left(  \frac{L(\widetilde{c})}{\lambda} \right)  ^{\frac{k}{k-1}} = 2 \frac{L^{\frac{k}{k-1}}(\widetilde{c})}{\lambda^{\frac{1}{k-1}}}.
\end{align*}

Now setting $\lambda=(\frac{1}{\delta})^{\frac{k-1}{k}}$, we get
\[
\delta\lambda+(II)\leq(2\cdot L^{\frac{k}{k-1}}(\widetilde{c})+1)\delta^{\frac{1}{k}}=q_{k}.
\]

\subsection{Additional Comments}
From Theorem 18 of \cite{luxburg2004distance}, for connected and centered sets
$\mathcal{S}_{X}$, with the following (tighter) definition for $\epsilon
(n,\rho,\zeta,K_{\lambda})$, Theorem \ref{thm:delta_sel} still holds.
\begin{align*}
\sqrt{\frac{\log(\frac{1}{\rho})}{2n} } + 4\zeta K_{\lambda} + \frac{8\sqrt
{2}K_{\lambda}}{\sqrt{n}}\int_{\zeta/4}^{2diam(\mathcal{S}_{X})} \sqrt{
\mathcal{N}(\mathcal{S}_{X},d_{X},\xi/2) \log2 + \log\left(  2\left\lceil
\frac{2diam(\mathcal{S}_{X})}{\xi} \right\rceil +1 \right)  }d{\xi}
\end{align*}

In particular, for $\mathcal{S}_{X} = [0,1]^{d}$ since $\mathcal{N}(\mathcal{S}_{X},d_{X},\xi) \leq\frac{H}{\xi^{d}}$ 
for some $H>0$, we get that for all $\zeta>0$ and $d>2$,
\begin{align}
\epsilon(n,\rho,\zeta, K_{\lambda}) \leq\sqrt{\frac{\log(\frac{1}{\rho})}{n}
}+ 4\zeta K_{\lambda}+ \frac{8\sqrt{2}K_{\lambda}}{\sqrt{n}} \left(  \frac{
2^{d/2}\sqrt{H\log2} \,\cdot\, (\zeta/4)^{(-d/2+1)}}{d/2-1} + 10 \cdot
diam(\mathcal{S}_{X})\right)  \label{eq_epsilon}
\end{align}

Finding the minimizing $\zeta$ results in
\begin{align*}
\epsilon(n,\rho,\zeta, K_{\lambda}) \leq\sqrt{\frac{\log(\frac{1}{\rho})}{n}
}+ 32 K_{\lambda} \left(  \frac{ 8\log2 \cdot H }{n}\right) ^{{1}/{d}} +
\frac{8 K_{\lambda}(8\log2\cdot H)^{1/d}}{(d/2 -1) \cdot n^{1/d}} +
\frac{(80\sqrt{2}) diam(\mathcal{S}_{X}) K_{\lambda}}{\sqrt{n}}
\end{align*}

The $n^{-1/d}$ factor (curse of dimensionality) aligns with
\cite{dudley1969speed,weed2017sharp}. However, (\ref{eq_epsilon}) also
suggests that fixing $\zeta$ results in an expression for $\epsilon$ which is
a constant plus an term of order $n^{-1/2}$. So, if one is willing to
compromise for a bias term $\Theta(\zeta)$, a convergence rate of $n^{-1/2}$
is attainable by fixing $\zeta$.

\section{Proof of Proposition \ref{prop:prop1}}

\label{apx_prf_prop1}

Using McDiarmid's inequality \cite{boucheron2013concentration}, it suffices to
show that $f(\cdot)$ satisfies the bounded difference condition.

Let $X_{1},\ldots, X_{n},X^{\prime}_{n}$ be i.i.d samples from the measure
$\mu_{0}$. Let $\mu_{n},\mu^{\prime}_{n}$ be the empirical measures associated
with $X_{1},\ldots,X_{n-1} ,X_{n}$ and $X_{1},\ldots, X_{n-1}, X^{\prime}_{n}$ respectively.

\begin{align*}
&  |f(X_{1},\ldots,X_{n}) - f(X_{1},\ldots,X^{\prime}_{n})| = \left|
\sup_{\alpha(\cdot) \in Lip(K_{\lambda})} \{ \mathbb{E}_{\mu_{n}}\alpha(W) +
\mathbb{E}_{\nu}\alpha^{h}_{\lambda}(Y) \} - \right. \\
&  \left.  \sup_{\alpha(\cdot) \in Lip(K_{\lambda})} \{ \mathbb{E}_{\mu^{\prime}_{n}}\alpha(W) + 
\mathbb{E}_{\nu}\alpha^{h}_{\lambda}(Y) \}
\right|  \leq\left|  \sup_{\alpha(\cdot) \in Lip(K_{\lambda})} \{
\mathbb{E}_{\mu_{n}}\alpha(W) - \mathbb{E}_{\mu^{\prime}_{n}}\alpha(W) \}
\right|  =\\
&  \left|  \sup_{\alpha(\cdot) \in Lip(K_{\lambda})} \frac{\alpha(X_{n}) -
\alpha(X^{\prime}_{n})}{n} \right|  \leq\left|  \sup_{\alpha(\cdot) \in
Lip(K_{\lambda})} \frac{\alpha(X_{n}) - \alpha(X^{\prime}_{n})}{n} \right|
\leq\\
&  \frac{K_{\lambda}}{n}d(X_{n},X^{\prime}_{n}) \leq\frac{K_{\lambda}}{n}diam(\mathcal{S}_{X})
\end{align*}

\section{Proof of Proposition \ref{prop:concentration_2}}

\label{apx_prf_prop_conc_2}

\begin{align*}
&  \min_{\pi\in\Pi_{W,Y}(\hat{\mu}_{n},\nu)} \mathbb{E}_{\pi}\left\{  -h(W,Y,\lambda
)\right\}  - \min_{\pi\in\Pi_{W,Y}(\mu_{0},\nu)} \mathbb{E}_{\pi}\left\{
-h(W,Y,\lambda)\right\}  =\\
&  \sup_{\alpha(\cdot) \in Lip(K_{\lambda})} \{ \mathbb{E}_{\hat{\mu}_{n}}\alpha(W) + \mathbb{E}_{\nu}\alpha^{h}_{\lambda}(Y) \} - \sup_{\alpha
(\cdot)\in Lip(K_{\lambda})} \{ \mathbb{E}_{\mu_{0}}\alpha(W) + \mathbb{E}_{\nu}\alpha^{h}_{\lambda}(Y) \}
\end{align*}
Using Proposition \ref{prop:prop1}, with probability at least $1 - \rho$, the above
expression is less than or equal to
\begin{align*}
&  \sqrt{\frac{\log(\frac{1}{\rho})}{2n} } + \mathbb{E}\left[  \sup
_{\alpha(\cdot) \in Lip(K_{\lambda})} \{ \mathbb{E}_{\hat{\mu}_{n}}\alpha(W) +
\mathbb{E}_{\nu}\alpha^{h}_{\lambda}(Y) \} - \sup_{\alpha(\cdot) \in
Lip(K_{\lambda})} \{ \mathbb{E}_{\mu_{0}}\alpha(W) + \mathbb{E}_{\nu}
\alpha^{h}_{\lambda}(Y) \} \right] \\
\leq\;  &  \sqrt{\frac{\log(\frac{1}{\rho})}{2n} } + \mathbb{E}\left[
\sup_{\alpha(\cdot) \in Lip(K_{\lambda})} \{ \mathbb{E}_{\hat{\mu}_{n}}\alpha(W) - \mathbb{E}_{\mu_{0}}\alpha(W) \} \right] \\
\overset{(a)}{\leq}\;  &  \sqrt{\frac{\log(\frac{1}{\rho})}{2n} } +
2R_{n}(Lip(K_{\lambda}))
\end{align*}

where (a) is an error bound result stated in \cite{luxburg2004distance}.
The proof of the other inequality is similar.

\section{Proof of Lemma \ref{approxm_lip}}

\label{apx_prf_lip_lemma}

\begin{align*}
&  f(x) = f(x) - \lambda\cdot d^{k}(x,x) \leq f(y_{x}) - \lambda\cdot
d^{k}(x,y_{x})\\
\Leftrightarrow &  \lambda\cdot d^{k}(x,y_{x}) \leq f(y_{x}) - f(x) \leq
L\cdot d(x,y_{x})\\
\Leftrightarrow &  (\lambda d^{k-1}(x,y_{x}) - L) d(x,y_{x}) \leq 0
\end{align*}
So for $k=1$ and $\lambda> L$, $d(y_{x},x) = 0 $. Also for $k>1$, $d(y_{x},x)
\leq {(\frac{L}{\lambda})}^{\frac{1}{k-1}}$.

\section{Additional Experiment Results for Section \ref{sec:wgan_experiments}}

Figures \ref{fig::fig4},\ref{fig::fig5} provide additional results for Section \ref{sec:wgan_experiments}. 
The initializations in these figures are different from Figures \ref{fig::fig1},\ref{fig::fig2}.

\begin{figure}[ptbh]
\centering
\begin{subfigure}[b]{0.45\textwidth}
\includegraphics[width=\textwidth]{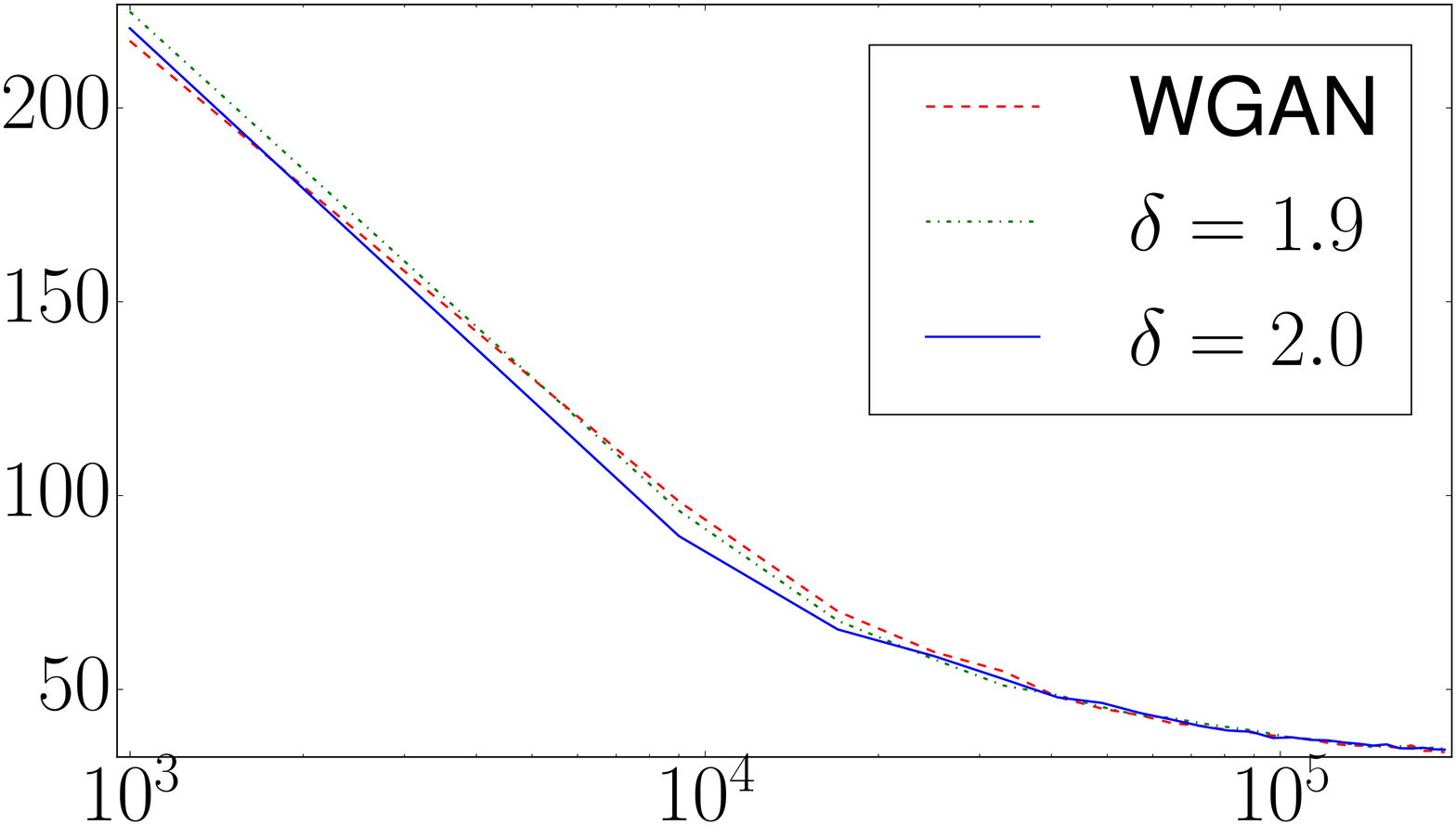}
\caption{}
\end{subfigure}
~ \begin{subfigure}[b]{0.45\textwidth}
\includegraphics[width=\textwidth]{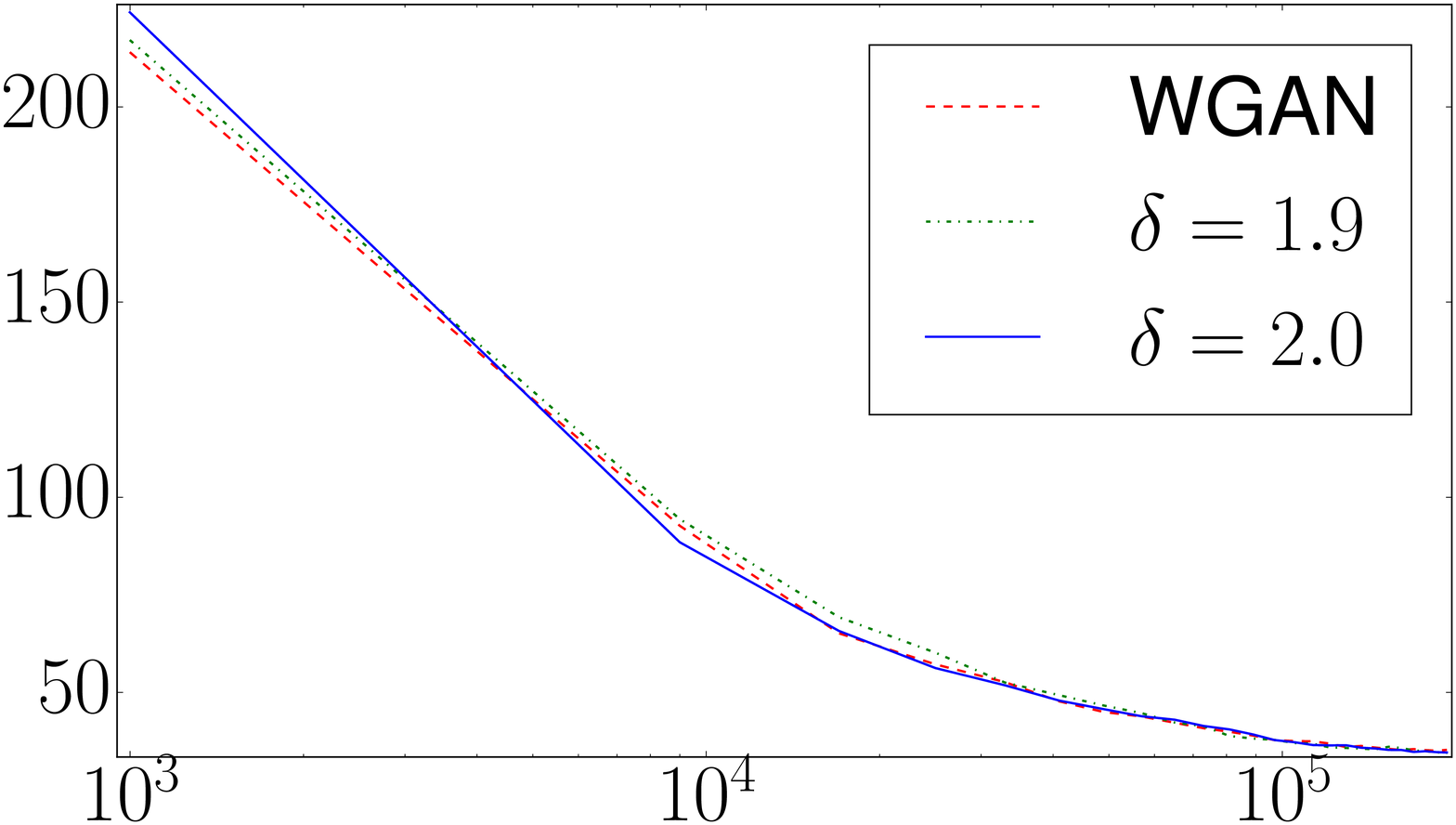}
\caption{}
\end{subfigure}
\caption{FID versus generator iteration on CIFAR10 for comparison of OTR WGAN
and WGAN. Each subfigure presents a different initialization. These initializations are different from Figure \ref{fig::fig1}. 
In the legend,
$\delta$ is the OTR WGAN parameter.}
\label{fig::fig4}
\end{figure}

\begin{figure}[ptbh]
\centering
\begin{subfigure}[b]{0.45\textwidth}
\includegraphics[width=\textwidth]{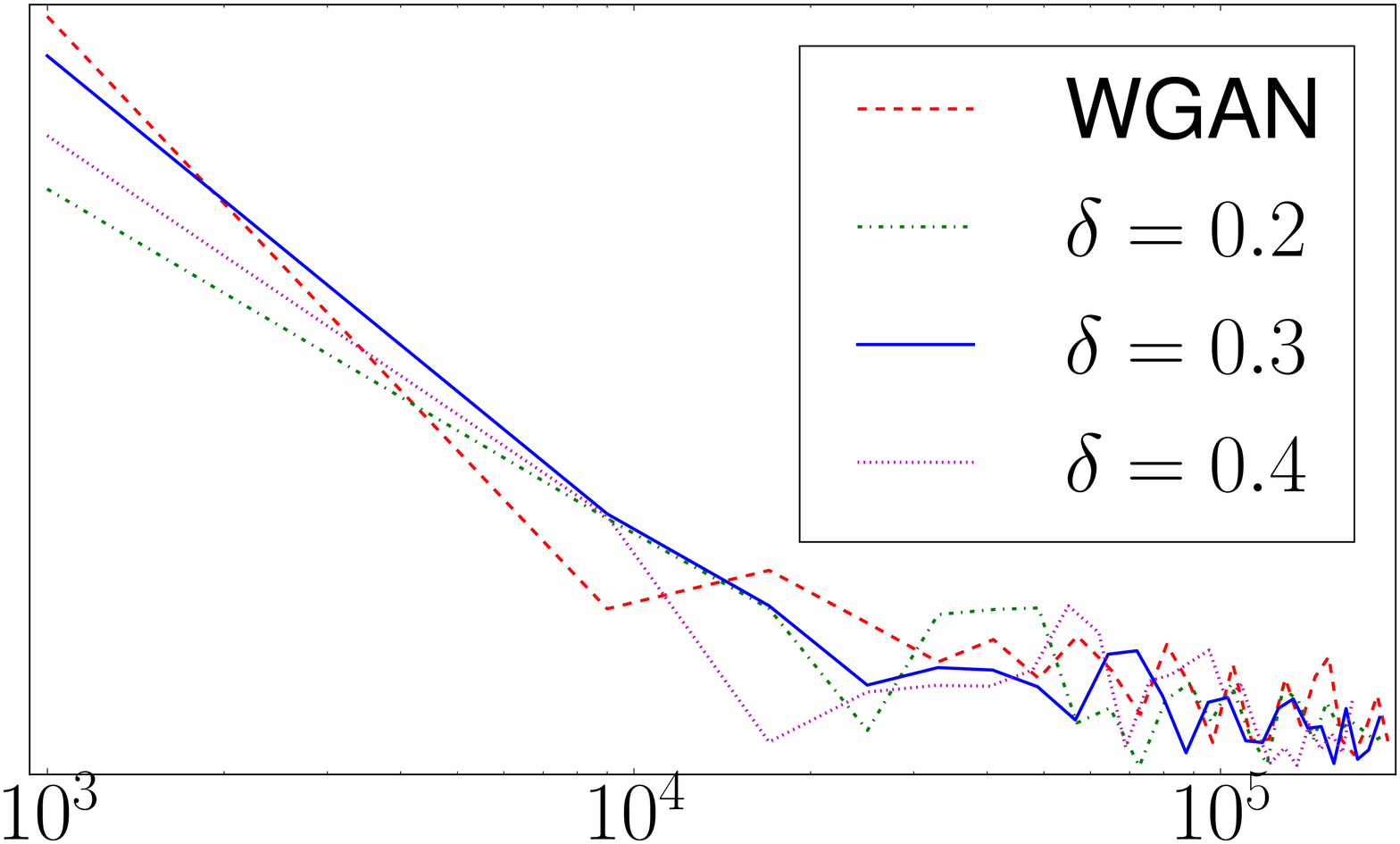}
\caption{}
\end{subfigure}
~ \begin{subfigure}[b]{0.45\textwidth}
\includegraphics[width=\textwidth]{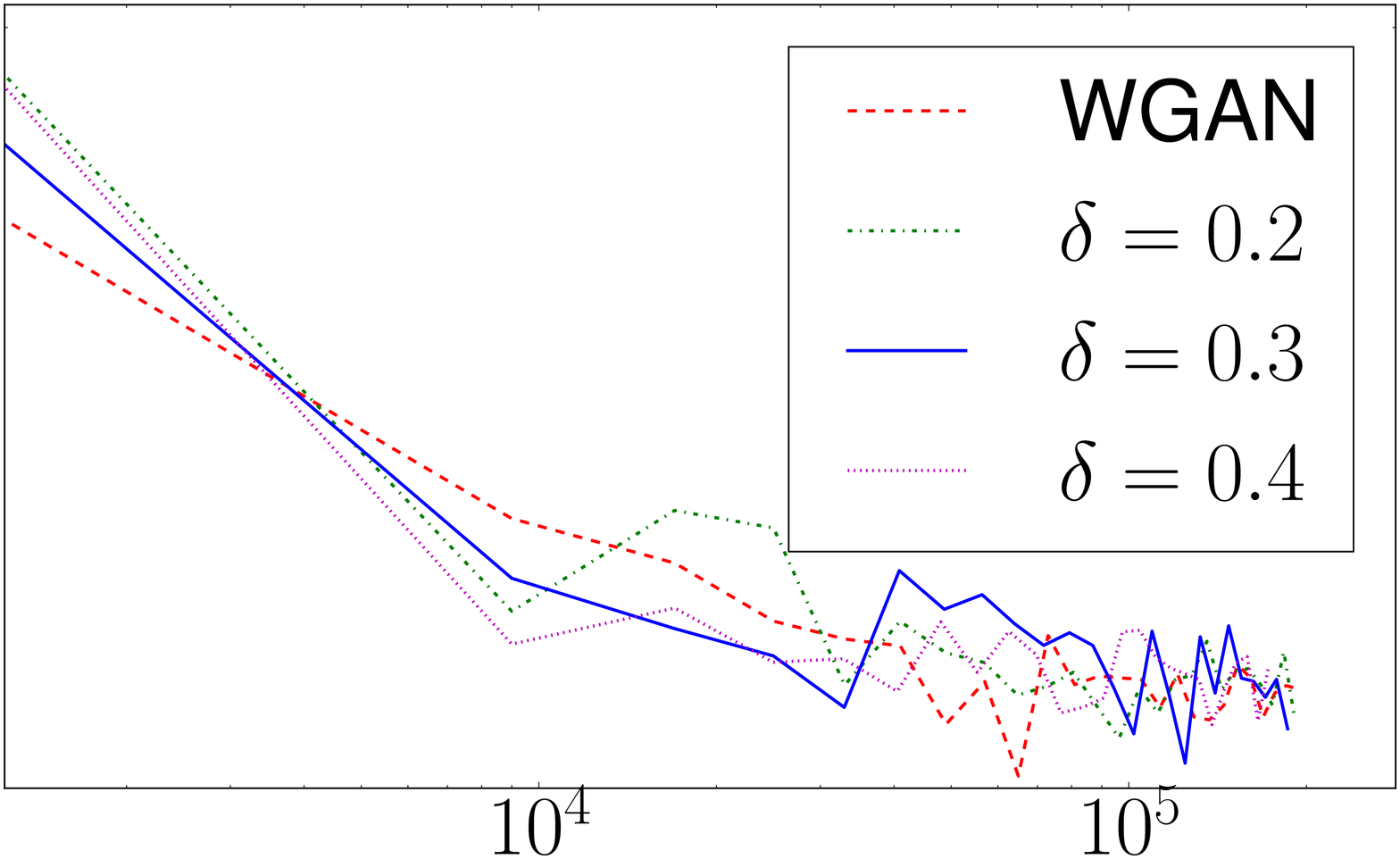}
\caption{}
\end{subfigure}
\caption{FID versus generator iteration on MNIST for comparison of OTR WGAN
and WGAN. Each subfigure presents a different initialization. 
These initializations are different from Figure \ref{fig::fig2}. 
In the legend, $\delta$ is the OTR WGAN parameter.}
\label{fig::fig5}
\end{figure}

\end{document}